\documentclass{elsarticle}

\usepackage{hyperref}
\hypersetup{colorlinks=true}

\usepackage{amssymb,amsmath}

\usepackage{algorithm,algorithmic}

\newtheorem{theorem}{Theorem}
\newtheorem{lemma}{Lemma}
\newtheorem{proposition}{Proposition}
\newtheorem{corollary}{Corollary}
\newdefinition{definition}{Definition}
\newproof{proof}{Proof}

\newcommand{\R}{\mathbb{R}}
\newcommand{\grad}{\nabla}
\newcommand{\norm}[1]{{\left\|{#1}\right\|}}
\newcommand{\Breg}{\mathcal{B}}

\DeclareMathOperator{\Regret}{Regret}
\DeclareMathOperator*{\Exp}{\mathbf{E}}
\DeclareMathOperator*{\argmin}{argmin}

\begin{document}

\begin{frontmatter}

\title{Scale-Free Online Learning\tnoteref{t1}}

\tnotetext[t1]{A preliminary version of this paper~\cite{Orabona-Pal-2015} was
               presented at ALT 2015.}

\author{Francesco Orabona\corref{aff1}}
\ead{francesco@orabona.com}
\cortext[aff1]{Work done while at Yahoo Research.}
\address{Stony Brook University, Stony Brook, NY 11794, USA}

\author{D\'avid P\'al\corref{cor1}}
\ead{dpal@yahoo-inc.com}
\cortext[cor1]{Corresponding author}

\address{Yahoo Research, 14th Floor, 229 West 43rd Street, New York, NY 10036, USA}

\begin{abstract}
We design and analyze algorithms for online linear optimization that have
optimal regret and at the same time do not need to know any upper or lower
bounds on the norm of the loss vectors. Our algorithms are instances of the
Follow the Regularized Leader (FTRL) and Mirror Descent (MD) meta-algorithms. We
achieve adaptiveness to the norms of the loss vectors by scale invariance, i.e.,
our algorithms make exactly the same decisions if the sequence of loss vectors
is multiplied by any positive constant.  The algorithm based on FTRL works for
any decision set, bounded or unbounded.  For unbounded decisions sets, this is
the first adaptive algorithm for online linear optimization with a non-vacuous
regret bound. In contrast, we show lower bounds on scale-free algorithms based
on MD on unbounded domains.

\end{abstract}

\end{frontmatter}

\section{Introduction}
\label{section:introduction}

Online Linear Optimization (OLO) is a problem where an algorithm repeatedly
chooses a point $w_t$ from a convex decision set $K$, observes an arbitrary, or
even adversarially chosen, loss vector $\ell_t$ and suffers the loss $\langle
\ell_t, w_t \rangle$.  The goal of the algorithm is to have a small cumulative
loss. The performance of an algorithm is evaluated by the so-called regret,
which is the difference of the cumulative losses of the algorithm and of the
(hypothetical) strategy that would choose in every round the same best point in
hindsight.

OLO is a fundamental problem in machine
learning~\cite{Cesa-Bianchi-Lugosi-2006, Rakhlin-Sridharan-2009,
Shalev-Shwartz-2011}.  Many learning problems can be directly phrased as OLO,
e.g., learning with expert advice~\cite{Littlestone-Warmuth-1994, Vovk-1998,
Freund-Schapire-1997, Cesa-Bianchi-Haussler-Helmbold-Schapire-Warmuth-1997}
and online combinatorial optimization~\cite{Kalai-Vempala-2005,
Helmbold-Warmuth-2009, Koolen-Warmuth-Kivinen-2010}. Other problems can be
reduced to OLO, e.g., online convex
optimization~\cite{Abernethy-Bartlett-Rakhlin-Tewari-2008},
\cite[Chapter~2]{Shalev-Shwartz-2011}, online classification
~\cite{Rosenblatt-1958, Freund-Schapire-1999} and
regression~\cite{Kivinen-Warmuth-1997},
~\cite[Chapters~11~and~12]{Cesa-Bianchi-Lugosi-2006}, multi-armed bandits
problems~\cite[Chapter~6]{Cesa-Bianchi-Lugosi-2006},
\cite{Abernethy-Hazan-Rakhlin-2008, Bubeck-Cesa-Bianchi-2012}, and batch and
stochastic optimization of convex functions~\cite{Nemirovski-Yudin-1983,
Bubeck-2015}.  Hence, a result in OLO immediately implies other results in all
these domains.

The adversarial choice of the loss vectors received by the algorithm is what
makes the OLO problem challenging. In particular, if an OLO algorithm commits
to an upper bound on the norm of future loss vectors, its regret can be made
arbitrarily large through an adversarial strategy that produces loss vectors
with norms that exceed the upper bound.

For this reason, most of the existing OLO algorithms receive as an input---or
implicitly assume---an upper bound $B$ on the norm of the loss vectors.  The
input $B$ is often disguised as the learning rate, the regularization
parameter, or the parameter of strong convexity of the regularizer.
However, these algorithms have two obvious drawbacks.

First, they do not come with any regret guarantee for sequences of loss vectors
with norms exceeding $B$. Second, on sequences of loss vectors with norms
bounded by $b \ll B$, these algorithms fail to have an optimal regret guarantee
that depends on $b$ rather than on $B$.

\renewcommand{\arraystretch}{1.8}

\begin{table}[t]
\fontsize{8}{8.2}\selectfont
\centering
\begin{tabular}{|p{3.6cm}|c|p{3.4cm}|c|}
\hline
\textbf{Algorithm} & \textbf{Decisions Set(s)} & \textbf{Regularizer(s)} & \textbf{Scale-Free} \\ \hline \hline
\textsc{Hedge} \cite{Freund-Schapire-1997} & Probability Simplex & Negative Entropy & No \\ \hline
\textsc{GIGA} \cite{Zinkevich-2003} & Any Bounded & $\frac{1}{2}\|w\|_2^2$ & No \\ \hline
\textsc{RDA} \cite{Xiao-2010} & \textbf{Any} & \textbf{Any Strongly Convex} & No \\ \hline
\textsc{FTRL-Proximal} \cite{McMahan-Streeter-2010,McMahan-2014} & Any Bounded & $\frac{1}{2}\|w\|_2^2 + $ any convex func.\footnotemark & \textbf{Yes} \\ \hline
\textsc{AdaGrad MD} \cite{Duchi-Hazan-Singer-2011} & Any Bounded & $\frac{1}{2}\|w\|_2^2 + $ any convex func. & \textbf{Yes} \\ \hline
\textsc{AdaGrad FTRL} \cite{Duchi-Hazan-Singer-2011} & \textbf{Any} & $\frac{1}{2}\|w\|_2^2 + $ any convex func. & No \\ \hline
\textsc{AdaHedge} \cite{de-Rooij-van-Erven-Grunwald-Koolen-2014} & Probability Simplex & Negative Entropy & \textbf{Yes} \\ \hline
\textsc{NAG} \cite{Ross-Mineiro-Langford-2013} & $\{u: \max_t \langle \ell_t, u\rangle \le C\}$ & $\frac{1}{2}\|w\|_2^2 $& Partially\footnotemark \\ \hline
\textsc{Scale invariant algorithms} \cite{Orabona-Crammer-Cesa-Bianchi-2014} & \textbf{Any} & $\frac{1}{2}\|w\|_p^2 + $ any convex func. \newline $1 < p \le 2$ & Partially\textsuperscript{\ref{footnote-label2}} \\ \hline
\textsc{Scale-free MD} \textbf{[this paper]} & $\sup_{u,v \in K} \Breg_f(u,v) < \infty$ & \textbf{Any Strongly Convex} & \textbf{Yes} \\ \hline
\textsc{SOLO FTRL} \textbf{[this paper]} & \textbf{Any} & \textbf{Any Strongly Convex} & \textbf{Yes} \\ \hline
\end{tabular}
\caption{Selected results for OLO. Best results in each column are in bold.
\label{table:results}
}
\end{table}

\addtocounter{footnote}{-2}

\stepcounter{footnote}\footnotetext{\label{footnote-label1} Even if, in
principle the FTRL-Proximal algorithm can be used with any proximal regularizer,
to the best of our knowledge a general way to construct proximal regularizers is
not known. The only proximal regularizer we are aware is based on the 2-norm.}

\stepcounter{footnote}\footnotetext{\label{footnote-label2} These algorithms
attempt to produce an invariant sequence of predictions $\langle w_t, \ell_t
\rangle$, rather than a sequence of invariant $w_t$.}

There is a clear practical need to design algorithms that adapt automatically to
the norms of the loss vectors.  A natural, yet overlooked, design method to
achieve this type of adaptivity is by insisting to have a \textbf{scale-free}
algorithm.  That is, with the same parameters, the sequence of decisions of the
algorithm does not change if the sequence of loss vectors is multiplied by a
positive constant.  The most important property of scale-free algorithms is that
both their loss and their regret scale linearly with the maximum norm of the
loss vector appearing in the sequence.

\subsection{Previous results}

The majority of the existing algorithms for OLO are based on two generic
algorithms: \textsc{Follow The Regularizer Leader} (\textsc{FTRL}) and
\textsc{Mirror Descent} (\textsc{MD}). \textsc{FTRL} dates back to the
potential-based forecaster in \cite[Chapter~11]{Cesa-Bianchi-Lugosi-2006} and
its theory was developed in \cite{Shalev-Shwartz-2007}.  The name \textsc{Follow
The Regularized Leader} comes from~\cite{Abernethy-Hazan-Rakhlin-2008}.
Independently, the same algorithm was proposed in~\cite{Nestorov-2009} for
convex optimization under the name \textsc{Dual Averaging} and rediscovered
in~\cite{Xiao-2010} for online convex optimization. Time-varying regularizers
were analyzed in~\cite{Duchi-Hazan-Singer-2011} and the analysis tightened in
\cite{Orabona-Crammer-Cesa-Bianchi-2014}. \textsc{MD} was originally proposed in
\cite{Nemirovski-Yudin-1983} and later analyzed in~\cite{Beck-Teboulle-2003} for
convex optimization. In the online learning literature it makes its first
appearance, with a different name, in~\cite{Kivinen-Warmuth-1997}.

Both \textsc{FTRL} and \textsc{MD} are parametrized by a function called
a \emph{regularizer}.
Based on different regularizers different algorithms with
different properties can be instantiated. A summary of algorithms for OLO is
presented in Table~\ref{table:results}.  All of them are instances of
\textsc{FTRL} or \textsc{MD}.

Scale-free versions of \textsc{MD} include \textsc{AdaGrad
MD}~\cite{Duchi-Hazan-Singer-2011}. However, the \textsc{AdaGrad MD} algorithm
has a non-trivial regret bounds only when the Bregman divergence associated with
the regularizer is bounded. In particular, since a bound on the Bregman
divergence implies that the decision set is bounded, the regret bound for
\textsc{AdaGrad MD} is vacuous for unbounded sets. In fact, as we show in
Section~\ref{subsection:mirror-descent-lower-bound}, \textsc{AdaGrad MD} and
similar algorithms based on \textsc{MD} incurs $\Omega(T)$ regret, in the worst
case, if the Bregman divergence is not bounded.

Only one scale-free algorithm based on \textsc{FTRL} was known. It is the
\textsc{AdaHedge}~\cite{de-Rooij-van-Erven-Grunwald-Koolen-2014} algorithm for
learning with expert advice, where the decision set is bounded. An algorithm
based on \textsc{FTRL} that is ``almost'' scale-free is \textsc{AdaGrad
FTRL}~\cite{Duchi-Hazan-Singer-2011}.  This algorithm fail to be scale-free due
to ``off-by-one'' issue; see~\cite{McMahan-2014} and the discussion in
Section~\ref{section:solo-ftrl}. Instead,
\textsc{FTRL-Proximal}~\cite{McMahan-Streeter-2010,McMahan-2014} solves the
off-by-one issue, but it requires proximal regularizers. In general, proximal
regularizers do not have a simple form and even the simple 2-norm case requires
bounded domains to achieve non-vacuous regret.

For unbounded decision sets no scale-free algorithm with a non-trivial regret
bound was known. Unbounded decision sets are practically important (see, e.g.,
\cite{Mcmahan-Holt-Sculley-2013}), since learning of large-scale linear models
(e.g., logistic regression) is done by gradient methods that can be reduced to
OLO with decision set $\R^d$.

\subsection{Overview of the Results}

We design and analyze two scale-free algorithms: \textsc{SOLO FTRL} and
\textsc{Scale-Free MD}.  A third one, \textsc{AdaFTRL}, is presented in
the Appendix. \textsc{SOLO FTRL} and \textsc{AdaFTRL} are based on
\textsc{FTRL}.  \textsc{AdaFTRL} is a generalization of
\textsc{AdaHedge}~\cite{de-Rooij-van-Erven-Grunwald-Koolen-2014} to arbitrary
strongly convex regularizers.  \textsc{SOLO FTRL} can be viewed as the
``correct'' scale-free version of the diagonal version of \textsc{AdaGrad
FTRL}~\cite{Duchi-Hazan-Singer-2011} generalized to arbitrary strongly convex
regularizers.  \textsc{Scale-Free MD} is based on \textsc{MD}. It is a
generalization of \textsc{AdaGrad MD}~\cite{Duchi-Hazan-Singer-2011} to
arbitrary strongly convex regularizers.  The three algorithms are presented in
Sections~\ref{section:solo-ftrl} and \ref{section:mirror-descent}, and
\ref{section:ada-ftrl}, respectively.

We prove that the regret of \textsc{SOLO FTRL} and \textsc{AdaFTRL} on bounded domains after $T$
rounds is bounded by $O (\sqrt{\sup_{v \in K} f(v) \sum_{t=1}^T\|\ell_t\|_*^2} )$
where $f$ is a non-negative regularizer that is $1$-strongly convex with respect
to a norm $\|\cdot\|$ and $\|\cdot\|_*$ is its dual norm. For \textsc{Scale-Free
MD}, we prove $O (\sqrt{\sup_{u,v \in K} B_f(u,v) \sum_{t=1}^T\|\ell_t\|_*^2} )$
where $B_f$ is the Bregman divergence associated with a $1$-strongly convex
regularizer $f$. In Section~\ref{section:lower-bound}, we show that the
$\sqrt{\sum_{t=1}^T \|\ell_t\|_*^2}$ term in the bounds is necessary by proving
a $\frac{D}{\sqrt{8}} \sqrt{\sum_{t=1}^T\|\ell_t\|_*^2}$ lower bound on the
regret of any algorithm for OLO for any decision set with diameter $D$ with
respect to the primal norm $\|\cdot\|$.

For \textsc{SOLO FTRL}, we prove that the regret against a competitor $u \in K$
is at most $O (f(u) \sqrt{\sum_{t=1}^T \|\ell_t\|_*^2} + \max_{t=1,2,\dots,T}
\|\ell_t\|_* \sqrt{T} )$.  As before, $f$ is a non-negative $1$-strongly convex
regularizer. This bound is non-trivial for any decision set, bounded or
unbounded.  The result makes \textsc{SOLO FTRL} the \textbf{first adaptive
algorithm for unbounded decision sets} with a non-trivial regret bound.

All three algorithms are \textbf{any-time}, i.e., they do not need to know the
number of rounds, $T$, in advance and the regret bounds hold for all $T$
simultaneously.

Our proof techniques rely on new homogeneous
inequalities (Lemmas~\ref{lemma:useful}, \ref{lemma:recurrence-solution})
which might be of independent interest.

Finally, in Section~\ref{subsection:mirror-descent-lower-bound}, we show
negative results for existing popular variants of \textsc{MD}. We show two
examples of decision sets and sequences of loss vectors of unit norm on which
these variants of \textsc{MD} have $\Omega(T)$ regret.  These results indicate
that \textsc{FTRL} is superior to \textsc{MD} in a worst-case sense.

\section{Notation and Preliminaries}
\label{section:preliminaries}

Let $V$ be a finite-dimensional\footnote{Many, but not all, of our
results can be extended to more general normed vector spaces.} real vector
space equipped with a norm $\|\cdot\|$. We denote by $V^*$ its dual vector
space.  The bi-linear map associated with $(V^*, V)$ is denoted by $\langle
\cdot, \cdot \rangle:V^* \times V \to \R$.  The dual norm of $\|\cdot\|$ is
$\|\cdot\|_*$.

In OLO, in each round $t=1,2,\dots$, the algorithm chooses a point $w_t$ in the
decision set $K \subseteq V$ and then the algorithm observes a loss vector
$\ell_t \in V^*$. The instantaneous loss of the algorithm in round $t$ is
$\langle \ell_t, w_t \rangle$. The cumulative loss of the algorithm after $T$
rounds is $\sum_{t=1}^T \langle \ell_t, w_t \rangle$.  The regret of the
algorithm with respect to a point $u \in K$ is
$$
\Regret_T(u) = \sum_{t=1}^T \langle \ell_t, w_t \rangle - \sum_{t=1}^T \langle \ell_t, u \rangle,
$$
and the regret with respect to the best point is $\Regret_T= \sup_{u \in K}
\Regret_T(u)$.  We assume that $K$ is a non-empty closed convex subset of $V$.
Sometimes we will assume that $K$ is also bounded. We denote by $D$ its
diameter with respect to $\|\cdot\|$, i.e., $D = \sup_{u,v \in K} \|u - v\|$.
If $K$ is unbounded, $D = +\infty$.

\subsection{Convex Analysis}

The \emph{Bregman divergence} of a convex differentiable function $f$ is
defined as $\Breg_f(u,v) = f(u) - f(v) - \langle \grad f(v), u - v \rangle$.
Note that $\Breg_f(u,v) \ge 0$ for any $u,v$ which follows directly from the
definition of convexity of $f$.

The \emph{Fenchel conjugate} of a function $f:K \to \R$ is the function
$f^*:V^* \to \R \cup \{+\infty\}$ defined as $f^*(\ell) = \sup_{w \in K} \left(
\langle \ell, w \rangle - f(w) \right)$.  The Fenchel conjugate of any function
is convex (since it is a supremum of affine functions) and satisfies
 the \emph{Fenchel-Young inequality}
$$
\forall w \in K, \ \forall \ell \in V^* \qquad \qquad
f(w) + f^*(\ell) \ge \langle \ell, w \rangle \; .
$$

Monotonicity of Fenchel conjugates follows easily from the definition: If
$f,g:K \to \R$ satisfy $f(w) \le g(w)$ for all $w \in K$ then $f^*(\ell) \ge
g^*(\ell)$ for every $\ell \in V^*$.

Given $\lambda > 0$, a function $f:K \to \R$ is called \emph{$\lambda$-strongly
convex} with respect to a norm $\|\cdot\|$ if and only if, for all $x,y \in K$,
$$
f(y) \ge f(x) + \langle \grad f(x), y - x \rangle + \frac{\lambda}{2}\|x - y\|^2 \; ,
$$
where $\grad f(x)$ is any subgradient of $f$ at the point $x$.

The following proposition relates the range of values of a strongly convex
function to the diameter of its domain. The proof can be found
in~\ref{section:definitions-proofs}.

\begin{proposition}[Diameter vs. Range]
\label{proposition:diameter-vs-range}
Let $K \subseteq V$ be a non-empty bounded closed convex set.  Let $D =
\sup_{u,v \in K} \|u - v\|$ be its diameter with respect to $\|\cdot\|$.  Let
$f:K \to \R$ be a non-negative lower semi-continuous function that is
$1$-strongly convex with respect to $\|\cdot\|$.  Then, $D \le \sqrt{8 \sup_{v
\in K} f(v)}$.
\end{proposition}

Fenchel conjugates and strongly convex functions have certain nice properties,
which we list in Proposition~\ref{proposition:conjugate-properties} below.

\begin{proposition}[Fenchel Conjugates of Strongly Convex Functions]
\label{proposition:conjugate-properties}
Let $K \subseteq V$ be a non-empty closed convex set with diameter
$D:=\sup_{u,v \in K} \|u-v\|$.  Let $\lambda > 0$, and let $f:K \to \R$ be a
lower semi-continuous function that is $\lambda$-strongly convex with respect
to $\|\cdot\|$.  The Fenchel conjugate of $f$ satisfies:
\begin{enumerate}

\item $f^*$ is finite everywhere and differentiable everywhere.

\item For any $\ell \in V^*$, $\grad f^*(\ell) = \argmin_{w \in K} \left( f(w) - \langle \ell, w \rangle \right)$.

\item For any $\ell \in V^*$, $f^*(\ell) + f(\grad f^*(\ell)) = \langle \ell, \grad f^*(\ell) \rangle$.

\item $f^*$ is $\frac{1}{\lambda}$-strongly smooth, i.e., for any $x,y \in
V^*$, $\Breg_{f^*}(x, y) \le \frac{1}{2\lambda} \|x - y\|_*^2$.

\item $f^*$ has $\frac{1}{\lambda}$-Lipschitz continuous gradients, i.e.,
for any $x,y \in V^*$,
$\|\grad f^*(x) - \grad f^*(y)\| \le \frac{1}{\lambda} \|x - y\|_*$.

\item $\Breg_{f^*}(x,y) \le D\|x-y\|_*$ for any $x,y \in V^*$.

\item $\|\grad f^*(x) - \grad f^*(y)\| \le D$ for any $x,y \in V^*$.

\item For any $c > 0$, $(cf(\cdot))^* = cf^*(\cdot/c)$.
\end{enumerate}
\end{proposition}

Except for properties 6 and 7, the proofs can be found
in~\cite{Shalev-Shwartz-2007}.  Property 6 is proven
in~\ref{section:definitions-proofs}. Property 7 trivially follows from property
2.

\begin{algorithm}[t]
\caption{\textsc{FTRL with Varying Regularizer}}
\label{algorithm:ftrl-varying-regularizer}
\begin{algorithmic}[1]
\REQUIRE Non-empty closed convex set $K \subseteq V$
\STATE Initialize $L_0 \leftarrow 0$
\FOR{$t=1,2,3,\dots$}
\STATE Choose a regularizer $R_t:K \to \R$
\STATE $w_t \leftarrow \argmin_{w \in K} \left( \langle L_{t-1}, w \rangle + R_t(w) \right)$
\STATE Predict $w_t$
\STATE Observe $\ell_t \in V^*$
\STATE $L_t \leftarrow L_{t-1} + \ell_t$
\ENDFOR
\end{algorithmic}
\end{algorithm}

\subsection{Generic FTRL with Varying Regularizer}
\label{section:generic-ftrl}

Two of our scale-free algorithms are instances of \textsc{FTRL} with
\emph{varying regularizers}, presented as
Algorithm~\ref{algorithm:ftrl-varying-regularizer}.  The algorithm is
paramatrized by a sequence $\{R_t\}_{t=1}^\infty$ of functions $R_t:K \to \R$
called \emph{regularizers}.  Each regularizer $R_t$ can depend on the past loss
vectors $\ell_1, \ell_2, \dots, \ell_{t-1}$ in an arbitrary way.  The following
lemma bounds its regret.

\begin{lemma}[Regret of \textsc{FTRL}]
\label{lemma:generic-regret-bound}
If the regularizers $R_1, R_2, \dots$ chosen by
Algorithm~\ref{algorithm:ftrl-varying-regularizer} are strongly convex and lower
semi-continuous, the algorithm's regret is upper bounded as
$$
\Regret_T(u) \le R_{T+1}(u) + R_1^*(0) + \sum_{t=1}^{T} \Breg_{R_t^*}(-L_t, -L_{t-1}) - R_t^*(-L_t) + R_{t+1}^*(-L_t) \; .
$$
\end{lemma}

The proof of the lemma can be found in~\cite{Orabona-Crammer-Cesa-Bianchi-2014}.
For completeness, we include it in~\ref{section:definitions-proofs}.

\subsection{Generic Mirror Descent with Varying Regularizer}

\begin{algorithm}[t]
\caption{\textsc{Mirror Descent with Varying Regularizer}}
\label{algorithm:mirror-descent-varying-regularizer}
\begin{algorithmic}[1]
\REQUIRE Non-empty closed convex set $K \subseteq V$
\STATE Choose a regularizer $R_0:K \to \R$
\STATE $w_1 \leftarrow \argmin_{w \in K} R_0(w)$
\FOR{$t=1,2,3,\dots$}
\STATE Predict $w_t$
\STATE Observe $\ell_t \in V^*$
\STATE Choose a regularizer $R_t:K \to \R$
\STATE $w_{t+1} \leftarrow \argmin_{w \in K} \left( \langle \ell_t, w \rangle + \Breg_{R_t}(w, w_t) \right)$
\ENDFOR
\end{algorithmic}
\end{algorithm}

\textsc{Mirror Descent} (MD) is a generic algorithm similar to \textsc{FTRL}
but quite different in the details. The algorithm is stated as
Algorithm~\ref{algorithm:mirror-descent-varying-regularizer}. The algorithm is
parametrized by a sequence $\{R_t\}_{t=0}^\infty$ of convex functions $R_t:K
\to \R$ called \emph{regularizers}. Each regularizer $R_t$ can depend on past
loss vectors $\ell_1, \ell_2, \dots, \ell_t$ in an arbitrary way. If $R_t$ is
not differentiable,\footnote{Note that this can happen even when $R_t$ is a
restriction of a differentiable function defined on a superset of $K$.  If $K$
is bounded and closed, $R_t$ fails to be differentiable at the boundary of $K$.
If $K$ is a subset of an affine subspace of a dimension smaller than the
dimension of $V$, then $R_t$ fails to be differentiable everywhere.} the
Bregman divergence, $\Breg_{R_t}(u,v) = R_t(u) - R_t(v) - \langle \grad R_t(v),
u - v \rangle$ needs to be defined. This is done by choosing a subgradient map
$\grad R_t:K \to V$, i.e., a function such that $\grad R_t(w)$ is a subgradient
of $R_t$ at any point $w$. If $R_t$ is a restriction of a differentiable function
$R'_t$, it is convenient to define $\grad R_t(w) = \grad R'_t(w)$ for all $w
\in K$. The following lemma bounds the regret of \textsc{MD}.

\begin{lemma}[Regret of \textsc{MD}]
\label{lemma:mirror-descent-regret}
Algorithm~\ref{algorithm:mirror-descent-varying-regularizer} satisfies, for any
$u \in K$,
$$
\Regret_T(u)
\le
\sum_{t=1}^T \langle \ell_t, w_t - w_{t+1} \rangle - \Breg_{R_t}(w_{t+1}, w_t) + \Breg_{R_t}(u,w_t) - \Breg_{R_t}(u, w_{t+1}) \; .
$$
\end{lemma}

The proof of the lemma can be found
in~\cite{Rakhlin-Sridharan-2009,Duchi-Shalev-Shwartz-Singer-Tewari-2010}.  For
completeness, we give a proof in~\ref{section:mirror-descent-proofs}.

\subsection{Per-Coordinate Learning}
\label{sec:per-coordinate}

An interesting class of algorithms proposed in~\cite{McMahan-Streeter-2010} and
\cite{Duchi-Hazan-Singer-2011} are based on so-called per-coordinate
learning rates.  As shown in \cite{Streeter-McMahan-2010}, any algorithm for
OLO can be used with per-coordinate learning rates as well.

Abstractly, we assume that the decision set is a Cartesian product $K=K_1
\times K_2 \times \dots \times K_d$ of a finite number of convex sets.  On each
factor $K_j$, $j=1,2,\dots,d$, we can run any OLO algorithm separately and we
denote by $\Regret_T^{(j)}(u_j)$ its regret with respect to $u_j \in K_j$. The
overall regret with respect to any $u=(u_1, u_2, \dots, u_d) \in K$ can be
written as
$$
\Regret_T(u) = \sum_{j=1}^d \Regret_T^{(j)}(u_j) \; .
$$

If the algorithm for each factor is scale-free, the overall algorithm is
clearly scale-free as well. Hence, even if not explicitly mentioned in the
text, any algorithm we present can be trivially transformed to a per-coordinate
version.

\section{\textsc{SOLO FTRL}}
\label{section:solo-ftrl}

In this section, we introduce our first scale-free algorithm; it will be based
on \textsc{FTRL}. The closest algorithm to a scale-free \textsc{FTRL} in the
existing literature is the \textsc{AdaGrad FTRL}
algorithm~\cite{Duchi-Hazan-Singer-2011}. It uses a regularizer on each
coordinate of the form
\begin{equation*}
R_t(w) = R(w) \left(\delta + \sqrt{\sum_{i=1}^{t-1} \|\ell_i\|_*^2} \right).
\end{equation*}
This kind of regularizer would yield a scale-free algorithm \emph{only} for
$\delta=0$. In fact, with this choice of $\delta$ it is easy to see that the
predictions $w_t$ in line 4 of
Algorithm~\ref{algorithm:ftrl-varying-regularizer} would be independent of the
scaling of the $\ell_t$.  Unfortunately, the regret bound
in~\cite{Duchi-Hazan-Singer-2011} becomes vacuous for such setting in the
unbounded case. In fact, it requires $\delta$ to be greater than $\|\ell_t\|_*$
for all time steps $t$, requiring knowledge of the future (see Theorem~5
in~\cite{Duchi-Hazan-Singer-2011}). In other words, despite of its name,
\textsc{AdaGrad FTRL} is not fully adaptive to the norm of the gradient
vectors. Similar considerations hold for the \textsc{FTRL-Proximal} in
\cite{McMahan-Streeter-2010,McMahan-2014}: The scale-free setting of the
learning rate is valid only in the bounded case.

One simple approach would be to use a doubling trick on $\delta$ in order to
estimate on the fly the maximum norm of the losses. Note that a naive strategy
would still fail because the initial value of $\delta$ should be data-dependent
in order to have a scale-free algorithm. Moreover, we would have to upper bound
the regret in all the rounds where the norm of the current loss is bigger than
the estimate. Finally, the algorithm would depend on an additional parameter,
the ``doubling'' power. Hence, even in the case one would prove a regret bound,
such strategy would give the feeling that \textsc{FTRL} needs to be ``fixed''
in order to obtain a scale-free algorithm.

In the following, we propose a much simpler and better approach.  We propose to
use Algorithm~\ref{algorithm:ftrl-varying-regularizer} with the regularizer
\begin{equation}
\label{equation:solo-ftrl-regularizer}
R_t(w) = R(w) \sqrt{\sum_{i=1}^{t-1} \|\ell_i\|_*^2} \; ,
\end{equation}
where $R:K \to \R$ is any strongly convex function. Through a refined analysis,
we show that this regularizer suffices to obtain an optimal regret bound for any
decision set, bounded or unbounded.  We call this variant \textsc{Scale-free
Online Linear Optimization FTRL} algorithm (\textsc{SOLO FTRL}).  Our main
result is Theorem~\ref{theorem:regret-solo-ftrl} below, which is proven in
Section~\ref{section:solo-ftrl-regret-bound}.

The regularizer~\eqref{equation:solo-ftrl-regularizer} does not uniquely define
the \textsc{FTRL} minimizer $w_t = \argmin_{w \in K} R_t(w)$ when
$\sqrt{\sum_{i=1}^{t-1} \|\ell_i\|_*^2}$ is zero. This happens if $\ell_1,
\ell_2, \dots, \ell_{t-1}$ are all zero (and in particular for $t=1$).  In that
case, we define $w_t = \argmin_{w \in K} R(w)$ which is consistent with
$w_t = \lim_{a \to 0^+}  \argmin_{w \in K} a R(w)$.

\begin{theorem}[Regret of \textsc{SOLO FTRL}]
\label{theorem:regret-solo-ftrl}
Suppose $K \subseteq V$ is a non-empty closed convex set.  Let $D = \sup_{u,v
\in K} \|u - v\|$ be its diameter with respect to a norm $\|\cdot\|$.  Suppose
that the regularizer $R:K \to \R$ is a non-negative lower semi-continuous
function that is $\lambda$-strongly convex with respect to $\|\cdot\|$. The
regret of \textsc{SOLO FTRL} satisfies
\begin{align*}
\Regret_T(u)
& \le \left( R(u) + \frac{2.75}{\lambda}\right) \sqrt{\sum_{t=1}^{T} \norm{\ell_t}_*^2}
+ 3.5 \min\left\{\frac{\sqrt{T-1}}{\lambda} , D\right\} \max_{t \le T} \|\ell_t\|_* \; .
\end{align*}
\end{theorem}

When $K$ is unbounded, we pay a penalty that scales as $\max_{t \le T}
\|\ell_t\|_* \sqrt{T}$, that has the same magnitude of the first term in the
bound. On the other hand, when $K$ is bounded, the second term is a constant and
we can choose the optimal multiple of the regularizer.  We choose $R(w) =
\lambda f(w)$ where $f$ is a $1$-strongly convex function and optimize
$\lambda$.  The result of the optimization is
Corollary~\ref{corollary:regret-solo-ftrl-bounded-set}.

\begin{corollary}[Regret Bound for Bounded Decision Sets]
\label{corollary:regret-solo-ftrl-bounded-set}
Suppose $K \subseteq V$ is a non-empty bounded closed convex set.  Suppose that
$f:K \to \R$ is a non-negative lower semi-continuous function that is
$1$-strongly convex with respect to $\|\cdot\|$. \textsc{SOLO FTRL} with
regularizer
$$
R(w) = \frac{f(w)\sqrt{2.75}}{\sqrt{\sup_{v \in K} f(v)}}
\quad \text{satisfies} \quad
\Regret_T \le 13.3 \sqrt{\sup_{v \in K} f(v) \sum_{t=1}^{T} \norm{\ell_t}_*^2} \; .
$$
\end{corollary}
\begin{proof}
Let $S = \sup_{v \in K} f(v)$. Theorem~\ref{theorem:regret-solo-ftrl} applied
to the regularizer $R(w) = \frac{c}{\sqrt{S}} f(w)$, together with
Proposition~\ref{proposition:diameter-vs-range} and a crude bound
$\max_{t=1,2,\dots,T} \|\ell_t\|_* \le \sqrt{\sum_{t=1}^T \|\ell_t\|_*^2}$,
give
$$
\Regret_T \le \left(c + \frac{2.75}{c}  + 3.5\sqrt{8} \right) \sqrt{S \sum_{t=1}^{T} \norm{\ell_t}_*^2} \; .
$$
We choose $c$ by minimizing $g(c) = c + \frac{2.75}{c} + 3.5\sqrt{8}$. Clearly,
$g(c)$ has minimum at $c = \sqrt{2.75}$ and has minimal value $g(\sqrt{2.75}) =
2\sqrt{2.75} + 3.5\sqrt{8} \le 13.3$.
\end{proof}

\subsection{Proof of Regret Bound for \textsc{SOLO FTRL}}
\label{section:solo-ftrl-regret-bound}

The proof of Theorem~\ref{theorem:regret-solo-ftrl} relies on an inequality
(Lemma~\ref{lemma:useful}).  Related and weaker inequalities, like
Lemma~\ref{lemma:sum-of-square-roots-inverses}, were proved
in~\cite{Auer-Cesa-Bianchi-Gentile-2002} and~\cite{Jaksch-Ortner-Auer-2010}. The
main property of this inequality is that on the right-hand side $C$ does
\emph{not} multiply the $\sqrt{\sum_{t=1}^T a_t^2}$ term.

\begin{lemma}[Useful Inequality]
\label{lemma:useful}
Let $C, a_1, a_2, \dots, a_T \ge 0$. Then,
$$
\sum_{t=1}^T \min \left\{ \frac{a_t^2}{\sqrt{\sum_{i=1}^{t-1} a_i^2}}, \ C a_t \right\}
\le 3.5 \, C \max_{t=1,2,\dots,T} a_t \ + \ 3.5 \sqrt{\sum_{t=1}^T a_t^2} \; .
$$
\end{lemma}
\begin{proof}
Without loss of generality, we can assume that $a_t > 0$ for all $t$. Since otherwise we
can remove all $a_t = 0$ without affecting either side of the inequality. Let $M_t = \max\{a_1, a_2, \dots, a_t\}$ and $M_0 = 0$.
We prove that for any $\alpha > 1$
$$
\min\left\{ \frac{a_t^2}{\sqrt{\sum_{i=1}^{t-1} a_i^2}}, C a_t \right\}
\le 2 \sqrt{1+\alpha^2} \left( \sqrt{\sum_{i=1}^t a_i^2} - \sqrt{\sum_{i=1}^{t-1} a_i^2} \right) + \frac{C\alpha( M_t  - M_{t-1})}{\alpha - 1}
$$
from which the inequality follows by summing over $t=1,2,\dots,T$ and choosing $\alpha = \sqrt{2}$.
The inequality follows by case analysis. If $a_t^2 \le \alpha^2 \sum_{i=1}^{t-1} a_i^2$, we have
\begin{multline*}
\min\left\{ \frac{a_t^2}{\sqrt{\sum_{i=1}^{t-1} a_i^2}}, C a_t \right\}
\le \frac{a_t^2}{\sqrt{\sum_{i=1}^{t-1} a_i^2}}
= \frac{a_t^2}{\sqrt{\frac{1}{1+\alpha^2} \left( \alpha^2 \sum_{i=1}^{t-1} a_i^2 + \sum_{i=1}^{t-1} a_i^2 \right)}} \\
\le \frac{a_t^2\sqrt{1+\alpha^2}}{\sqrt{ a_t^2 + \sum_{i=1}^{t-1} a_i^2 }}
= \frac{a_t^2\sqrt{1+\alpha^2}}{\sqrt{\sum_{i=1}^t a_i^2}}
\le 2\sqrt{1+\alpha^2} \left( \sqrt{\sum_{i=1}^t a_i^2} - \sqrt{\sum_{i=1}^{t-1} a_i^2} \right)
\end{multline*}
where we have used $x^2/\sqrt{x^2+y^2} \le 2(\sqrt{x^2+y^2} - \sqrt{y^2})$ in the last step.
On the other hand, if $a_t^2 > \alpha^2 \sum_{t=1}^{t-1} a_i^2$, we have
\begin{multline*}
\min\left\{ \frac{a_t^2}{\sqrt{\sum_{i=1}^{t-1} a_i^2}}, \ C a_t \right\}
\le C a_t
= C \frac{\alpha a_t  - a_t}{\alpha - 1}
\le \frac{C}{\alpha - 1} \left( \alpha a_t  - \alpha \sqrt{\sum_{i=1}^{t-1} a_i^2} \right) \\
= \frac{C\alpha}{\alpha - 1} \left( a_t  - \sqrt{\sum_{i=1}^{t-1} a_i^2} \right)
\le \frac{C\alpha}{\alpha - 1} \left( a_t  - M_{t-1} \right)
= \frac{C\alpha}{\alpha - 1} \left( M_t  - M_{t-1} \right)
\end{multline*}
where we have used that $a_t = M_t$ and $\sqrt{\sum_{i=1}^{t-1} a_i^2} \ge M_{t-1}$.
\end{proof}

\begin{lemma}[{{\cite[Lemma~3.5]{Auer-Cesa-Bianchi-Gentile-2002}}}]
\label{lemma:sum-of-square-roots-inverses}
Let $a_1, a_2, \dots, a_T$ be non-negative real numbers. If $a_1 > 0$ then,
$$
\sum_{t=1}^T \frac{a_t}{\sqrt{\sum_{i=1}^t a_i}} \le 2 \sqrt{\sum_{t=1}^T a_t} \; .
$$
\end{lemma}
For completeness, a proof of Lemma~\ref{lemma:sum-of-square-roots-inverses} is
in~\ref{section:solo-ftrl-proofs}.

\begin{proof}[Proof of Theorem~\ref{theorem:regret-solo-ftrl}]
Let $\eta_t = \frac{1}{\sqrt{\sum_{i=1}^{t-1} \|\ell_i\|_*^2}}$, hence $R_t(w)
= \frac{1}{\eta_t} R(w)$.  We assume without loss of generality that
$\|\ell_t\|_* > 0$ for all $t$, since otherwise we can remove all rounds $t$
where $\ell_t = 0$ without affecting the regret and the predictions of the
algorithm on the remaining rounds.  By Lemma~\ref{lemma:generic-regret-bound},
\begin{align*}
\Regret_T(u)
& \le \frac{1}{\eta_{T+1}} R(u) + \sum_{t=1}^T \left( \Breg_{R_t^*}(-L_t, -L_{t-1}) - R_t^*(-L_t) + R_{t+1}^*(-L_t) \right) \; .
\end{align*}
We upper bound the terms of the sum in two different ways.
First, by Proposition~\ref{proposition:conjugate-properties}, we have
$$
\Breg_{R_t^*}(-L_t, -L_{t-1}) - R_t^*(-L_t) + R_{t+1}^*(-L_t)
\le \Breg_{R_t^*}(-L_t, -L_{t-1})
\le \frac{\eta_t \|\ell_t\|_*^2}{2\lambda} \; .
$$
Second, we have
\begin{align*}
\Breg_{R_t^*}&(-L_t, -L_{t-1}) - R_t^*(-L_t) + R_{t+1}^*(-L_t) \\
& = \Breg_{R_{t+1}^*}(-L_t, -L_{t-1}) + R^*_{t+1}(-L_{t-1}) - R_t^*(-L_{t-1}) \\
& \qquad + \langle \nabla R_t^*(-L_{t-1})-\nabla R_{t+1}^*(-L_{t-1}), \ell_t \rangle  \\
& \le \frac{\eta_{t+1} \|\ell_t\|_*^2}{2\lambda} + \| \nabla R_t^*(-L_{t-1})-\nabla R_{t+1}^*(-L_{t-1})\| \cdot \|\ell_t\|_* \\
& = \frac{\eta_{t+1} \|\ell_t\|_*^2}{2\lambda} + \| \nabla R^*(- \eta_{t} L_{t-1})-\nabla R^*(- \eta_{t+1} L_{t-1})\| \cdot \|\ell_t\|_* \\
& \le \frac{\eta_{t+1} \|\ell_t\|_*^2}{2\lambda} + \min\left\{\frac{1}{\lambda} \|L_{t-1}\|_* \left(\eta_{t} - \eta_{t+1} \right), D\right\} \|\ell_t\|_* \; ,
\end{align*}
where in the first inequality we have used the fact that $R^*_{t+1}(-L_{t-1})
\le R_t^*(-L_{t-1})$, H\"older's inequality, and
Proposition~\ref{proposition:conjugate-properties}.  In the second inequality
we have used properties 5 and 7 of
Proposition~\ref{proposition:conjugate-properties}. Using the definition of
$\eta_{t+1}$ we have
\begin{align*}
\frac{\|L_{t-1}\|_* (\eta_{t} -\eta_{t+1})}{\lambda}
\le \frac{ \|L_{t-1}\|_*}{\lambda \sqrt{\sum_{i=1}^{t-1} \|\ell_i\|_*^2}}
\le \frac{ \sum_{i=1}^{t-1} \|\ell_i\|_*}{\lambda \sqrt{\sum_{i=1}^{t-1} \|\ell_i\|_*^2}}
\le \frac{\sqrt{t-1}}{\lambda}
\le \frac{\sqrt{T-1}}{\lambda}.
\end{align*}
Denoting by $H=\min\left\{\frac{\sqrt{T-1}}{\lambda},D\right\}$ we have
\begin{align*}
& \Regret_T(u)
\le \frac{1}{\eta_{T+1}} R(u) + \sum_{t=1}^T \min\left\{ \frac{\eta_t \|\ell_t\|_*^2}{2\lambda}, \ H \|\ell_t\|_* + \frac{\eta_{t+1} \|\ell_t\|_*^2}{2\lambda}  \right\} \\
& \le \frac{1}{\eta_{T+1}} R(u) + \frac{1}{2\lambda} \sum_{t=1}^T  \eta_{t+1} \|\ell_t\|_*^2 + \frac{1}{2\lambda} \sum_{t=1}^T \min\left\{ \eta_t \|\ell_t\|_*^2, \ 2 \lambda H \|\ell_t\|_* \right\} \\
& = \frac{1}{\eta_{T+1}} R(u) + \frac{1}{2\lambda} \sum_{t=1}^T  \frac{\|\ell_t\|_*^2}{\sqrt{\sum_{i=1}^t \|\ell_i\|_*^2}} + \frac{1}{2 \lambda} \sum_{t=1}^T \min\left\{ \frac{\|\ell_t\|_*^2}{\sqrt{\sum_{i=1}^{t-1} \|\ell_i\|_*^2}}, \ 2 \lambda H \|\ell_t\|_* \right\} \; .
\end{align*}
We bound each of the three terms separately. By definition of $\eta_{T+1}$, the
first term is $\frac{1}{\eta_{T+1}} R(u) = R(u) \sqrt{\sum_{t=1}^T
\|\ell_t\|_*^2}$.  We upper bound the second term using
Lemma~\ref{lemma:sum-of-square-roots-inverses} as
$$
\frac{1}{2\lambda} \sum_{t=1}^T  \frac{\|\ell_t\|_*^2}{\sqrt{\sum_{i=1}^t \|\ell_i\|_*^2}}
\le \frac{1}{\lambda} \sqrt{\sum_{t=1}^T \|\ell_t\|_*^2} \; .
$$
Finally, by Lemma~\ref{lemma:useful} we upper bound the third term as
$$
\frac{1}{2 \lambda} \sum_{t=1}^T \min\left\{ \frac{\|\ell_t\|_*^2}{\sqrt{\sum_{i=1}^{t-1} \|\ell_i\|_*^2}}, \ 2 \lambda \|\ell_t\|_* H \right\}
\le 3.5 H \max_{t \le T} \|\ell_t\|_* + \frac{1.75}{\lambda} \sqrt{\sum_{t=1}^T \|\ell_t\|_*^2} \; .
$$
Putting everything together gives the stated bound.
\end{proof}

\section{Scale-Free Mirror Descent}
\label{section:mirror-descent}

In this section, we analyze scale-free version of \textsc{Mirror Descent}.
Our algorithm uses the regularizer
\begin{equation}
\label{equation:scale-free-mirror-descent}
R_t(w) = R(w) \sqrt{\sum_{i=1}^t \|\ell_i\|_*^2} \; ,
\end{equation}
where $R:K \to \R$ an arbitrary strongly convex function.  As for \textsc{SOLO
FTRL}, it is easy to see that such regularizer gives rise to predictions $w_t$
that are scale-free.  We call the resulting algorithm \textsc{Scale-Free MD}.
Similar to \textsc{SOLO FTRL}, the regularizer
\eqref{equation:scale-free-mirror-descent} does not uniquely define the
\textsc{MD} minimizer $w_{t+1} = \argmin_{w \in K} \left(\langle \ell_t, w
\rangle + \Breg_{R_t}(w, w_t) \right)$ when $\sqrt{\sum_{i=1}^t
\|\ell_i\|_*^2}$ is zero.  This happens when the loss vectors $\ell_1, \ell_2,
\dots, \ell_t$ are all zero. In this case, we define $w_{t+1} = \argmin_{w \in
K} R(w)$ which agrees with $w_{t+1} = \lim_{a \to 0^+} \argmin_{w \in K} a
\Breg_{R}(w, w_t)$.  Similarly, $w_1 = \argmin_{w \in K} R(w)$.

Per-coordinate version of \textsc{Scale-Free MD} with regularizer $R(w) =
\frac{1}{2}\norm{w}_2^2$ is exactly the same algorithm as the diagonal version
of \textsc{AdaGrad MD}~\cite{Duchi-Hazan-Singer-2011}.

The theorem below upper bounds the regret of \textsc{Scale-Free MD} (see also
\cite{Duchi-Hazan-Singer-2011, Duchi-Shalev-Shwartz-Singer-Tewari-2010,
Rakhlin-Sridharan-2013}).  The proof is in~\ref{section:mirror-descent-proofs}.

\begin{theorem}[Regret of Scale-Free Mirror Descent]
\label{theorem:regret-scale-free-mirror-descent}
Suppose $K \subseteq V$ is a non-empty closed convex set. Suppose that $R:K
\to \R$ is a $\lambda$-strongly convex function with respect to a norm
$\|\cdot\|$.  \textsc{Scale-Free MD} with regularizer $R$ satisfies
for any $u \in K$,
$$
\Regret_T(u)
\le
\left( \frac{1}{\lambda} + \sup_{v \in K} \Breg_R(u,v) \right) \sqrt{\sum_{t=1}^T \|\ell_t\|_*^2} \; .
$$
\end{theorem}

We choose the regularizer $R(w) = \lambda f(w)$ where $f$ is a $1$-strongly convex
function and optimize $\lambda$. The result is the following Corollary. Its
proof is trivial.

\begin{corollary}[Regret of Scale-Free Mirror Descent]
\label{corollary:regret-scale-free-mirror-descent}
Suppose $K \subseteq V$ is a non-empty bounded closed convex set.  Suppose
that $f:K \to \R$ is a $1$-strongly convex function with respect to a norm
$\|\cdot\|$.  \textsc{Scale-Free MD} with regularizer
$$
R(w) = \frac{f(w)}{\displaystyle \sqrt{\sup_{u,v \in K} \Breg_f(u,v)}}
\quad \text{satisfies} \quad %
\Regret_T \le 2 \sqrt{\sup_{u,v \in K} \Breg_f(u,v) \sum_{t=1}^T \|\ell_t\|_*^2} \; .
$$
\end{corollary}

The regret bound for \textsc{Scale-Free MD} in the
Corollary~\ref{corollary:regret-scale-free-mirror-descent} depends on
$\sup_{u,v \in K} \Breg_f(u,v)$. In contrast, the regret bound for
\textsc{SOLO FTRL} in Corollary~\ref{corollary:regret-solo-ftrl-bounded-set}
depend on $\sup_{u \in K} f(u)$.  Similarly, the regret bound in
Theorem~\ref{theorem:regret-scale-free-mirror-descent} for \textsc{Scale-Free
MD} depends on $\sup_{v \in K} \Breg_R(u,v)$ and the regret bounds in
Theorem~\ref{theorem:regret-solo-ftrl}
for \textsc{SOLO FTRL} depend on $R(u)$. It is not hard to
show that
\begin{equation}
\label{equation:regularizer-vs-divergence}
\forall u \in K \qquad R(u) \le \sup_{v \in K} \Breg_R(u,v) \; ,
\end{equation}
provided that at the minimizer $v^* = \argmin_{v \in K} R(v)$ both $R(v^*)$ and
$\grad R(v^*)$ are zero. Indeed, in that case, $R(u) = \Breg_R(u,v^*) \le
\sup_{v \in K} \Breg_R(u,v)$.

The assumption $R(v^*) = 0$ and $\grad R(v^*) = 0$
are easy to achieve by adding an affine function to the regularizer:
$$
R'(u) = R(u) - \langle \grad R(v^*), u - v^* \rangle - R(v^*) \; .
$$
The regularizer $R'$ has the same parameter of strong convexity as $R$, the
associated Bregman divergences $\Breg_{R'}$ and $\Breg_{R}$ are equal, $R'$ and
$R$ have the same minimizer $v^*$, and $R'(v^*)$ and $\grad R'(v^*)$ are both
zero.

Thus, inequality \eqref{equation:regularizer-vs-divergence} implies
that---ignoring constant factors---the regret bound for \textsc{Scale-Free MD}
is inferior to the regret bound for \textsc{SOLO FTRL}.
In fact, it is not hard to come up with examples where $R(u)$ is finite whereas
$\sup_{v \in K} \Breg_R(u,v)$ is infinite. We mention two such examples. The
first example is $R(w) = \frac{1}{2}\|w\|_2^2$ defined on the whole space $V$,
where for any $u \in V$, $R(u)$ is a finite value but $\sup_{v \in K}
\Breg_R(u,v) = \sup_{v \in V} \frac{1}{2}\|u-v\|_2^2 = +\infty$. The second
example is the shifted negative entropy regularizer $R(w) = \ln(d) +
\sum_{j=1}^d w_j \ln w_j$ defined on the $d$-dimensional probability simplex $K
= \{ w \in \R^d ~:~ w_j \ge 0, \sum_{j=1}^d w_j = 1 \}$, where for any $u \in
K$, $R(u)$ is finite and in fact lies in the interval $[0, \ln d]$ but
$\sup_{v \in K} \Breg_R(u,v) = \sup_{v \in K} \sum_{j=1}^d u_j \ln(u_j/v_j) = +
\infty$. We revisit these examples in the following subsection.

\subsection{Lower Bounds for Scale-Free Mirror Descent}
\label{subsection:mirror-descent-lower-bound}

The bounds in Theorem~\ref{theorem:regret-scale-free-mirror-descent} and
Corollary~\ref{corollary:regret-scale-free-mirror-descent} are vacuous when
$\Breg_R(u,v)$ is not bounded. One might wonder if the assumption that
$\Breg_R(u,v)$ is bounded is necessary in order for \textsc{Scale-Free MD} to
have a sublinear regret. We show necessity of this assumption on two
counter-examples.  In these counter-examples, we consider strongly convex
regularizers $R$ such that $\Breg_R(u,v)$ is not bounded and we construct
sequences of loss vectors $\ell_1, \ell_2, \dots, \ell_T$ such that
$\|\ell_1\|_* = \|\ell_2\|_* = \dots = \|\ell_T\|_* = 1$ and \textsc{Scale-Free
MD} has regret $\Omega(T)$ or worse.

The first counter-example is stated as
Theorem~\ref{theorem:first-counter-example} below; our proof is
in~\ref{section:mirror-descent-proofs}. The decision set is the whole space
$K=V$ and the regularizer is $R(w) = \frac{1}{2}\|w\|_2^2$. Note that $R(w)$ is
$1$-strongly convex with respect to $\|\cdot\|_2$ and the dual norm of
$\|\cdot\|_2$ is $\|\cdot\|_2$. The corresponding Bregman divergence is
$\Breg_R(u,v) = \frac{1}{2}\|u-v\|_2^2$. The counter-example constructs a
sequence of unit-norm loss vectors in the one-dimensional subspace spanned by
the first vector of the standard orthnormal basis.  On such a sequence, both
versions of \textsc{AdaGrad MD} as well as  \textsc{Scale-Free MD} are identical
to gradient descent with step size $1/\sqrt{t}$, i.e., they are identical
Zinkevich's \textsc{Generalized Infinitesimal Gradient Ascent} (\textsc{GIGA})
algorithm~\cite{Zinkevich-2003}. Hence the lower bound applies to all these
algorithms.

\begin{theorem}[First Counter-Example]
\label{theorem:first-counter-example}
Suppose $K = V$. For any $T \ge 42$, there exists a sequence of loss vectors
$\ell_1, \ell_2, \dots, \ell_T \in V^*$ such that $\|\ell_1\|_2 = \|\ell_2\|_2
= \dots = \|\ell_T\|_2 = 1$ and \textsc{Scale-Free MD} with
regularizer $R(w) = \frac{1}{2}\|w\|_2^2$, \textsc{GIGA},
and both versions of \textsc{AdaGrad MD} satisfy
$$
\Regret_T(0) \ge \frac{T^{3/2}}{20} \; .
$$
\end{theorem}

The second counter-example is stated as
Theorem~\ref{theorem:second-counter-example} below; our proof is
in~\ref{section:mirror-descent-proofs}.  The decision set is the
$d$-dimensional probability simplex $K = \{ w \in \R^d ~:~ w_j \ge 0,
\sum_{j=1}^d w_j = 1 \}$ and the regularizer is the negative entropy $R(w) =
\sum_{j=1}^d w_j \ln w_j$.  Negative entropy is $1$-strongly convex with
respect to $\|\cdot\|_1$ and the dual norm of $\|\cdot\|_1$ is
\mbox{$\|\cdot\|_\infty$}.  The corresponding Bregman divergence is the
Kullback-Leibler divergence $\Breg_R(u,v) = \sum_{j=1}^d u_j \ln(u_j/v_j)$.
Note that despite that negative entropy is upper- and lower-bounded,
Kullback-Leibler divergence can be arbitrarily large.

\begin{theorem}[Second Counter-Example]
\label{theorem:second-counter-example}
Let $d \ge 2$, let $V = \R^d$, and let $K = \{ w \in V ~:~ w_j \ge 0,
\sum_{j=1}^d w_j = 1 \}$ be the $d$-dimensional probability simplex.  For any
$T \ge 120$, there exists a sequence of loss vectors $\ell_1, \ell_2, \dots,
\ell_T \in V^*$ such that $\|\ell_1\|_\infty = \|\ell_2\|_\infty = \dots =
\|\ell_T\|_\infty = 1$ and \textsc{Scale-Free MD} with regularizer
$R(w) = \sum_{j=1}^d w_j \ln w_j$ satisfies
$$
\Regret_T \ge \frac{T}{6} \; .
$$
\end{theorem}

\section{Lower Bound}
\label{section:lower-bound}

We show a lower bound on the worst-case regret of any algorithm for OLO. The
proof, presented in~\ref{section:lower-bound-proof}, is a standard
probabilistic argument.

\begin{theorem}[Lower Bound]
\label{theorem:simple-lower-bound}
Let $K \subseteq V$ be any non-empty bounded closed convex subset. Let $D =
\sup_{u,v \in K} \|u - v\|$ be the diameter of $K$. Let $A$ be any (possibly
randomized) algorithm for OLO on $K$. Let $T$ be any non-negative integer and
let $a_1, a_2, \dots, a_T$ be any non-negative real numbers.  There exists a
sequence of vectors $\ell_1, \ell_2, \dots, \ell_T$ in the dual vector space
$V^*$ such that $\|\ell_1\|_* = a_1, \|\ell_2\|_* = a_2, \dots, \|\ell_T\|_* =
a_T$ and the regret of algorithm $A$ satisfies
\begin{equation}
\label{equation:simple-lower-bound}
\Regret_T \ge \frac{D}{\sqrt{8}} \sqrt{\sum_{t=1}^T\|\ell_t\|_*^2} \; .
\end{equation}
\end{theorem}

The upper bounds on the regret, which we have proved for our algorithms, have
the same dependency on the norms of the loss vectors.  However, a gap remains
between the lower bound and the upper bounds.

The upper bound on regret of \textsc{SOLO FTRL} is of
the form $O(\sqrt{\sup_{v \in K} f(v) \sum_{t=1}^T \|\ell_t\|_*^2})$ where $f$
is any $1$-strongly convex function with respect to $\|\cdot\|$.  The same
upper bound is also achieved by \textsc{FTRL} with a constant learning rate
when $\sum_{t=1}^T \|\ell_t\|_*^2$ is known
upfront \cite[Chapter 2]{Shalev-Shwartz-2011}.  The lower bound is
$\Omega(D\sqrt{\sum_{t=1}^T \|\ell_t\|_*^2})$.

The gap between $D$ and $\sqrt{\sup_{v \in K} f(v)}$ can be substantial.  For
example, if $K$ is the probability simplex in $\R^d$ and $f(w) = \ln(d) +
\sum_{j=1}^d w_j \ln w_j$ is the shifted negative entropy, the
$\|\cdot\|_1$-diameter of $K$ is $2$, $f$ is non-negative and $1$-strongly
convex with respect to $\|\cdot\|_1$, but $\sup_{v \in K} f(v) = \ln(d)$.  On
the other hand, if the norm $\|\cdot\|_2 = \sqrt{\langle \cdot, \cdot \rangle}$
arises from an inner product $\langle \cdot, \cdot \rangle$, the lower bound
matches the upper bounds within a constant factor.  The reason is that for any
$K$ with $\|\cdot\|_2$-diameter $D$, the function $f(w) = \frac{1}{2} \|w -
w_0\|_2^2$, where $w_0$ is an arbitrary point in $K$, is $1$-strongly convex
with respect to $\|\cdot\|_2$ and satisfies that $\sqrt{\sup_{v \in K} f(v)}
\le D$. This leads to the following open problem (posed also
in~\cite{Kwon-Mertikopoulos-2014}):
\begin{quotation}
\noindent
\emph{Given a bounded convex set $K$ and a norm $\|\cdot\|$, construct a
non-negative function $f:K \to \R$ that is $1$-strongly convex with respect to
$\|\cdot\|$ and minimizes $\sup_{v \in K} f(v)$.}
\end{quotation}
As shown in~\cite{Srebro-Sridharan-Tewari-2011}, the existence of $f$ with
small $\sup_{v \in K} f(v)$ is equivalent to the existence of an algorithm for
OLO with $\widetilde O(\sqrt{T \sup_{v \in K} f(v)})$ regret assuming
$\|\ell_t\|_* \le 1$.  The $\widetilde O$ notation hides a polylogarithmic
factor in $T$.

\section{Conclusions}
\label{section:conclusions}

We have investigated scale-free algorithms for online linear optimization and
we have shown that the scale-free property leads to algorithms which have optimal
regret and do not need to know or assume \textbf{anything} about the sequence
of loss vectors. In particular, the algorithms do not assume any upper or lower
bounds on the norms of the loss vectors or the number of rounds.

We have designed a scale-free algorithm based on \textsc{Follow The Regularizer
Leader}. Its regret with respect to any competitor $u$ is
$$
O \left(f(u) \sqrt{\sum_{t=1}^T \|\ell_t\|_*^2}
+ \min\{\sqrt{T}, D\} \max_{t=1,2,\dots,T} \|\ell_t\|_* \right) \; ,
$$
where $f$ is any non-negative $1$-strongly convex
function defined on the decision set and $D$ is the diameter of the decision
set. The result makes sense even when the decision set is unbounded.

A similar, but weaker result holds for a scale-free algorithm based on
\textsc{Mirror Descent}. However, we have also shown this algorithm to be strictly
weaker than algorithms based on \textsc{Follow The Regularizer Leader}.
Namely, we gave examples of regularizers for which the scale-free version of
\textsc{Mirror Descent} has $\Omega(T)$ regret or worse.

We have proved an $\frac{D}{\sqrt{8}} \sqrt{\sum_{t=1}^T \|\ell\|_*^2}$ lower
bound on the regret of any algorithm for any decision set with diameter $D$.

Notice that with the regularizer $f(u) = \frac{1}{2}\|u\|_2^2$ the regret of
\textsc{SOLO FTRL} depends quadratically on the norm of the competitor
$\|u\|_2$. There exist non-scale-free algorithms \cite{McMahan-Streeter-2012,
McMahan-Abernethy-2013, Orabona-2013, McMahan-Orabona-2014, Orabona-2014,
Orabona-Pal-2016} that have only a $O(\|u\|_2 \sqrt{\log \|u\|_2})$ or
$O(\|u\|_2 \log \|u\|_2)$ dependency.  These algorithms assume an a priori bound
on the norm of the loss vectors. Recently, an algorithm that adapts to norms of
loss vectors and has a $O(\|u\|_2 \log \|u\|_2)$ dependency was
proposed~\cite{Cutkosky-Boahen-2016}. However, the trade-off between the
dependency on $\norm{u}_2$ and the adaptivity to the norms of the loss vectors
still remains to be explored.

\section*{Acknowledgments}
We thank an anonymous reviewer for suggesting a simpler proof of
Lemma~\ref{lemma:recurrence-solution}.

\section*{References}

\bibliographystyle{elsarticle-num-names}
\bibliography{biblio}

\appendix
\section{Proofs for Preliminaries}
\label{section:definitions-proofs}

\begin{proof}[Proof of Proposition~\ref{proposition:diameter-vs-range}]
Let $S = \sup_{u \in K} f(u)$ and $v^* = \argmin_{v \in K} f(v)$. The minimizer
$v^*$ is guaranteed to exist by lower semi-continuity of $f$ and compactness of
$K$.  The optimality condition for $v^*$ and $1$-strong convexity of $f$ imply that
for any $u \in K$,
$$
S
\ge f(u) - f(v^*)
\ge f(u) - f(v^*) - \langle \grad f(v^*), u - v^* \rangle
\ge \frac{1}{2}\|u - v^*\|^2 \; .
$$
In other words, $\|u - v^*\| \le \sqrt{2S}$. By the triangle inequality,
$$
D = \sup_{u,v \in K} \|u - v\| \le \sup_{u,v \in K} \left( \|u - v^*\| + \|v^ * - v\| \right) \le 2\sqrt{2S} = \sqrt{8S} \; .
$$
\end{proof}

\begin{proof}[Proof of Property 6 of Proposition~\ref{proposition:conjugate-properties}]
To bound $\Breg_{f^*}(x,y)$ we add a non-negative divergence term
$\Breg_{f^*}(y,x)$.
\begin{align*}
\Breg_{f^*}(x, y)
& \le \Breg_{f^*}(x,y) + \Breg_{f^*}(y,x)
= \langle x - y, \grad f^*(x) - \grad f^*(y) \rangle \\
& \le \|x - y\|_* \cdot \| \grad f^*(x) - \grad f^*(y) \|
\le D \|x - y\|_* \; ,
\end{align*}
where we have used H\"older's inequality and property 7 of the Proposition.
\end{proof}

\begin{proof}[Proof of Lemma~\ref{lemma:generic-regret-bound}]
By the Fenchel-Young inequality,
\begin{align*}
\sum_{t=1}^T \left( R_{t+1}^*(-L_t) - R_t^*(-L_{t-1}) \right)
& = R_{T+1}^*(-L_T) - R_1^*(0) \\
& \ge - \langle L_T, u \rangle - R_{T+1}(u)  - R_1^*(0) \\
& = - R_{T+1}(u)  - R_1^*(0) - \sum_{t=1}^T \langle \ell_t, u \rangle \; .
\end{align*}
We add $\sum_{t=1}^T \langle \ell_t, w_t \rangle$ to both sides and we obtain
$\Regret_T(u)$ on the right side. After rearrangement of the terms, we get an upper bound
on the regret:
\begin{align*}
\Regret_T(u)
& = \sum_{t=1}^T \langle \ell_t, w_t \rangle - \sum_{t=1}^T \langle \ell_t, u \rangle \\
& \le R_{T+1}(u) + R_1^*(0) + \sum_{t=1}^T \left( R_{t+1}^*(-L_t) - R_t^*(-L_{t-1}) + \langle \ell_t, w_t \rangle \right) \; .
\end{align*}
By Proposition~\ref{proposition:conjugate-properties}, property 2, we have $w_t =
\grad R_t^*(-L_{t-1})$ and therefore we can rewrite the sum in last expression as
\begin{align*}
& \sum_{t=1}^T R_{t+1}^*(-L_t) - R_t^*(-L_{t-1}) + \langle \ell_t, w_t \rangle \\
& = \sum_{t=1}^T R_{t+1}^*(-L_t) - R_t^*(-L_{t-1}) + \langle \ell_t, \grad R_t^*(-L_{t-1}) \rangle \\
& = \sum_{t=1}^T R_t^*(-L_t) - R_t^*(-L_{t-1}) + \langle \ell_t, \grad R_t^*(-L_{t-1}) \rangle - R_t^*(-L_t) + R_{t+1}^*(-L_t) \\
& = \sum_{t=1}^T \Breg_{R_t^*}(-L_t, -L_{t-1}) - R_t^*(-L_t) + R_{t+1}^*(-L_t) \; .
\end{align*}
This finishes the proof.
\end{proof}

\section{\textsc{AdaFTRL}}
\label{section:ada-ftrl}

In this section, we show that it is possible to derive a scale-free algorithm
different from \textsc{SOLO FTRL}. We generalize the \textsc{AdaHedge}
algorithm~\cite{de-Rooij-van-Erven-Grunwald-Koolen-2014} to the OLO setting,
showing that it retains its scale-free property.  We call the resulting
algorithm \textsc{AdaFTRL}. The analysis is very general and based on general
properties of strongly convex functions, rather than specific properties of the
entropic regularizer as in the original analysis of \textsc{AdaHedge}.

Assume that $K$ is bounded and that $R:K \to \R$ is a strongly convex lower
semi-continuous function bounded from above.  We instantiate
Algorithm~\ref{algorithm:ftrl-varying-regularizer} with the sequence of
regularizers
\begin{equation}
\label{equation:ada-ftrl}
R_t(w) = \Delta_{t-1} R(w)
\quad \text{where}
\quad \Delta_{t}=\sum_{i=1}^{t} \Delta_{i-1} \Breg_{R^*}\left(- \frac{L_i}{\Delta_{i-1}}, -\frac{L_{i-1}}{\Delta_{i-1}}\right) \; .
\end{equation}

The sequence $\{\Delta_t\}_{t=0}^\infty$ is non-negative and non-decreasing.
Also, $\Delta_t$ as a function of $\ell_1, \ell_2, \dots, \ell_t$ is positive
homogeneous of degree one, making the algorithm scale-free.

If $\Delta_{i-1} = 0$, we define $\Delta_{i-1}
\Breg_{R^*}(\frac{-L_i}{\Delta_{i-1}}, \frac{-L_{i-1}}{\Delta_{i-1}})$ as
$\lim_{a \to 0^+} a \Breg_{R^*}(\frac{-L_i}{a}, \frac{-L_{i-1}}{a})$ which
always exists and is finite; see Lemma~\ref{lemma:limit-bregman-divergence}
in~\ref{section:limits}.  Similarly, when $\Delta_{t-1} = 0$, we define $w_t =
\argmin_{w \in K} \langle L_{t-1}, w \rangle$ where ties among minimizers are
broken by taking the one with the smallest value of $R(w)$, which is unique due
to strong convexity. As we show in Lemma~\ref{lemma:prediction-limit-existence}
in~\ref{section:limits}, this is the same as $w_t = \lim_{a \to 0^+} \argmin_{w
\in K} (\langle L_{t-1}, w \rangle + aR(w))$.

Our main result is an $O(\sqrt{\sum_{t=1}^T \|\ell_t\|_*^2})$ upper bound on the
regret of the algorithm after $T$ rounds, without the need to know beforehand an
upper bound on $\|\ell_t\|_*$.  We prove the theorem
in~\ref{section:ada-ftrl-regret-bound}.

\begin{theorem}[Regret Bound]
\label{theorem:ada-ftrl-regret-bound}
Suppose $K \subseteq V$ is a non-empty bounded closed convex set. Let $D =
\sup_{x,y \in K} \|x - y\|$ be its diameter with respect to a norm $\|\cdot\|$.
Suppose that the regularizer $R:K \to \R$ is a non-negative lower
semi-continuous function that is $\lambda$-strongly convex with respect to
$\|\cdot\|$ and is bounded from above.  The regret of \textsc{AdaFTRL}
satisfies
$$
\Regret_T(u)
\le \sqrt{3} \max\left\{D, \frac{1}{\sqrt{2\lambda}} \right\}
    \sqrt{\sum_{t=1}^T \|\ell_t\|_*^2} \left(1 + R(u) \right) \; .
$$
\end{theorem}

The regret bound can be optimized by choosing the optimal multiple of the
regularizer.  Namely, we choose regularizer of the form $\lambda f(w)$ where
$f(w)$ is $1$-strongly convex and optimize over $\lambda$. The result of the
optimization is the following corollary.

\begin{corollary}[Regret Bound]
\label{corollary:ada-ftrl-regret-bound}
Suppose $K \subseteq V$ is a non-empty bounded closed convex set. Suppose $f:K
\to \R$ is a non-negative lower semi-continuous function that is $1$-strongly
convex with respect to $\|\cdot\|$ and is bounded from above.  The regret of
\textsc{AdaFTRL} with regularizer
$$
R(w) = \frac{f(w)}{16 \cdot \sup_{v \in K} f(v)}
\qquad \text{satisfies} \qquad
\Regret_T \le
5.3 \sqrt{\sup_{v \in K} f(v) \sum_{t=1}^T \|\ell_t\|_*^2} \; .
$$
\end{corollary}
\begin{proof}
Let $S = \sup_{v \in K} f(v)$. Theorem~\ref{theorem:ada-ftrl-regret-bound}
applied to the regularizer $R(w) = \frac{c}{S} f(w)$ and
Proposition~\ref{proposition:diameter-vs-range} gives
$$
\Regret_T \le \sqrt{3}(1 + c) \max\left\{\sqrt{8}, \frac{1}{\sqrt{2c}} \right\} \sqrt{S \sum_{t=1}^T \|\ell_t\|_*^2} \; .
$$
It remains to find the minimum of $g(c) = \sqrt{3}(1 + c) \max\{\sqrt{8},
1/\sqrt{2c}\}$.  The function $g$ is strictly convex on $(0, \infty)$ and has
minimum at $c=1/16$ and $g(\frac{1}{16}) = \sqrt{3}(1+\frac{1}{16})\sqrt{8} \le
5.3$.
\end{proof}

\subsection{Proof of Regret Bound for \textsc{AdaFTRL}}
\label{section:ada-ftrl-regret-bound}

\begin{lemma}[Initial Regret Bound]
\label{lemma:initial-regret-bound}
\textsc{AdaFTRL} satisfies, for any $u \in K$ and any $T \ge 0$,
$$
\Regret_T(u) \le \left(1 + R(u) \right) \Delta_T \; .
$$
\end{lemma}

\begin{proof}
Recall from \eqref{equation:ada-ftrl} that $R_t(w) = \Delta_{t-1} R(w)$. Since $R$ is non-negative,
$\{R_t\}_{t=1}^\infty$ is non-decreasing.  Hence, $R_t^*(\ell) \ge
R_{t+1}^*(\ell)$ for every $\ell \in V^*$ and thus $R_t^*(-L_t) -
R_{t+1}^*(-L_t) \ge 0$.  So, by Lemma~\ref{lemma:generic-regret-bound},
\begin{equation}
\label{equation:regret-bound-inequality}
\Regret_T(u) \le R_{T+1}(u) + R_1^*(0) + \sum_{t=1}^{T} \Breg_{R_t^*}(-L_t, -L_{t-1}) \; .
\end{equation}
Technically, \eqref{equation:regret-bound-inequality} is not justified since
$R_t$ might not be strongly convex. This happens when $\Delta_{t-1} = 0$. In
order to justify \eqref{equation:regret-bound-inequality}, we consider a
different algorithm that initializes $\Delta_0 = \epsilon$ where $\epsilon >
0$; that ensures that $\Delta_{t-1} > 0$ and $R_t$ is strongly convex.
Applying  Lemma~\ref{lemma:generic-regret-bound} and then taking limit
$\epsilon \to 0$, yields \eqref{equation:regret-bound-inequality}.

Since, $\Breg_{R_t^*}(u,v) = \Delta_{t-1} \Breg_{R^*}(\frac{u}{\Delta_{t-1}},
\frac{v}{\Delta_{t-1}})$ by definition of Bregman divergence and property 8 of
Proposition~\ref{proposition:conjugate-properties}, we have $\sum_{t=1}^T
\Breg_{R_t^*}(-L_t, -L_{t-1}) = \Delta_T$.
\end{proof}

\begin{lemma}[Recurrence]
\label{lemma:gap-recurrence}
Let $D = \sup_{u, v \in K} \|u -v\|$ be the diameter of $K$.  The sequence
$\{\Delta_t\}_{t=1}^\infty$ generated by \textsc{AdaFTRL} satisfies for any $t
\ge 1$,
$$
\Delta_t \le \Delta_{t-1} + \min \left\{ D\|\ell_t\|_*, \ \frac{\|\ell_t\|_*^2}{2\lambda \Delta_{t-1}} \right\} \; .
$$
\end{lemma}

\begin{proof}
By definition, $\Delta_t$ satisfies the recurrence
$$
\Delta_t = \Delta_{t-1} + \Delta_{t-1} \Breg_{R^*}\left( - \frac{L_t}{\Delta_{t-1}}, - \frac{L_{t-1}}{\Delta_{t-1}} \right) \; .
$$
Using parts 4 and 6 of Proposition~\ref{proposition:conjugate-properties}, we
can upper bound $\Breg_{R^*}\left( - \frac{L_t}{\Delta_{t-1}}, -
\frac{L_{t-1}}{\Delta_{t-1}} \right)$ with two different quantities.  Taking the
minimum of the two quantities finishes the proof.
\end{proof}

The recurrence of Lemma~\ref{lemma:gap-recurrence} can be simplified. Defining
$$
a_t = \|\ell_t\|_* \max \left\{D, \frac{1}{\sqrt{2 \lambda}} \right\} \; ,
$$
we get a recurrence
$$
\Delta_t \le \Delta_{t-1} + \min \left\{ a_t, \ \frac{a_t^2}{\Delta_{t-1}} \right\} \; .
$$
The next lemma solves this recurrence, by giving an explicit upper bound on
$\Delta_T$ in terms of $a_1, a_2, \dots, a_T$.

\begin{lemma}[Solution of the Recurrence]
\label{lemma:recurrence-solution}
Let $\{a_t\}_{t=1}^\infty$ be any sequence of non-negative real numbers.
Suppose that $\{\Delta_t\}_{t=0}^\infty$ is a sequence of non-negative
real numbers satisfying
$$
\Delta_0 = 0 \qquad \text{and} \qquad
\Delta_t \le \Delta_{t-1} + \min \left\{ a_t, \ \frac{a_t^2}{\Delta_{t-1}} \right\} \quad \text{for any $t \ge 1$} \; .
$$
Then, for any $T \ge 0$,
$$
\Delta_T \le \sqrt{3 \sum_{t=1}^T a_t^2} \; .
$$
\end{lemma}
\begin{proof}
Observe that
\[
\Delta_T^2
= \sum_{t=1}^T \Delta_t^2 - \Delta_{t-1}^2
= \sum_{t=1}^T (\Delta_t - \Delta_{t-1})^2 + 2 (\Delta_t - \Delta_{t-1}) \Delta_{t-1} \; .
\]
We bound each term in the sum separately.
The left term of the minimum inequality in the definition of $\Delta_t$ gives
\[
(\Delta_t - \Delta_{t-1})^2 \le a_t^2,
\]
while the right term gives
\[
2 (\Delta_t - \Delta_{t-1}) \Delta_{t-1} \le 2 a_t^2 \; .
\]
So, we conclude
\[
\Delta_T^2
\le
3 \sum_{t=1}^T a_t^2\; .
\]
\end{proof}

Theorem~\ref{theorem:ada-ftrl-regret-bound}
follows from
Lemmas~\ref{lemma:initial-regret-bound},~\ref{lemma:gap-recurrence}~and~\ref{lemma:recurrence-solution}.

\section{Limits}
\label{section:limits}

In this section, we show that prediction of \textsc{AdaFTRL} is correctly
defined when the regularizer is multiplied by zero.

\begin{lemma}[Prediction for Zero Regularizer]
\label{lemma:prediction-limit-existence}
Let $K$ be non-empty bounded closed convex subset of a finite dimensional
normed real vector space $(V, \|\cdot\|)$.  Let $R:K \to \R$ be strictly convex
and lower semi-continuous, and let $L \in V^*$. The limit
\begin{equation}
\label{equation:limit}
\lim_{\eta \to +\infty}
\argmin_{w \in K} \left( \langle L, w \rangle + \frac{1}{\eta} R(w) \right)
\end{equation}
exists and it is equal to the unique minimizer of $R(w)$ over the set (of minimizers)
$$
\left\{ w \in K ~:~ \langle L, w \rangle = \inf_{v \in K} \langle L, v \rangle \right\} \; .
$$
\end{lemma}

Before we give the proof, we illustrate the lemma on a simple
example.  Let $K = [-1,1]^2$ be a closed square in $\R^2$ and let
$R(w) = \|w\|_2^2$. Let $L = (1,0)$. The minimizers are
$$
\argmin_{w \in K} \langle L, w \rangle = \{ (-1,y) ~:~ y \in [-1,1] \} \; .
$$
The minimizer with the smallest value of $R(w)$ is $(-1,0)$. Hence the lemma
implies that
$$
\lim_{\eta \to +\infty} \argmin_{w \in K}
\left( \langle L, w \rangle + \frac{1}{\eta} \|w\|_2^2 \right) = (-1, 0) \; .
$$

\begin{proof}[Proof of Lemma~\ref{lemma:prediction-limit-existence}]
Without loss of generality, we can assume that $R(w)$ is non-negative for any
$w \in K$.  For otherwise, we can replace $R(w)$ with $R'(w) = R(w) - \inf_{v
\in K} R(v)$.

Since $K$ is a non-empty bounded closed convex subset of a finite dimensional
normed vector space, it is compact and $r^* = \min_{w \in K} \langle L, w
\rangle$ exists and is attained at some $w \in K$. Consider the hyperplane
$$
H = \{ w \in V ~:~ \langle L, w \rangle = r^* \} \; .
$$
The intersection $H \cap K$ is a non-empty compact convex set.
Let
$$
v^* = \argmin_{v \in K \cap H} R(v) \; .
$$
The existence of $v^*$ follows from compactness of $H \cap K$ and lower
semi-continuity of $R(v)$.  Uniqueness of $v^*$ follows from strict convexity
of $R(v)$. We show that the limit (\ref{equation:limit}) equals $v^*$.

By the definition of $H$,
\begin{equation}
\label{equation:limit-is-minimizer}
v^* \in \argmin_{w \in K} \ \langle L, w \rangle \; .
\end{equation}
Let $S = \{ w \in K ~:~ R(w) \le R(v^*) \}$. Since $R(w)$ is lower
semi-continuous $S$ is closed. Since $R(w)$ is strictly convex, $S \cap H =
\{v^*\}$.

For any $\eta > 0$, let
$$
w(\eta) = \argmin_{w \in K} \left( \langle L, w \rangle + \frac{1}{\eta} R(w) \rangle \right) \; .
$$
We prove that $w(\eta) \in S$. Indeed, by optimality of $v^*$ and $w(\eta)$,
$$
\frac{1}{\eta} R(w(\eta))  + \langle L, w(\eta) \rangle
\le
\frac{1}{\eta} R(v^*)  + \langle L, v^* \rangle
\le
\frac{1}{\eta} R(v^*)  + \langle L, w(\eta) \rangle
$$
and hence $R(w(\eta)) \le R(v^*)$.

By non-negativity of $R$ and optimality of $w(\eta)$ we have
$$
\langle L, w(\eta) \rangle
\le \langle L, w(\eta) \rangle + \frac{1}{\eta} R(w(\eta))
\le \langle L, v^* \rangle + \frac{1}{\eta} R(v^*) \; .
$$
Taking the limit $\eta \to +\infty$, we see that
$$
\lim_{\eta \to +\infty} \langle L, w(\eta) \rangle
\le\lim_{\eta \to +\infty} \left( \langle L, v^* \rangle + \frac{1}{\eta} R(v^*) \right)
= \langle L, v^* \rangle \; .
$$
From~(\ref{equation:limit-is-minimizer}) we have
$\langle L, v^* \rangle \le \langle L, w \rangle$ for any $w$, and therefore
\begin{equation}
\label{equation:limit-equality}
\lim_{\eta \to +\infty} \langle L, w(\eta) \rangle = \langle L, v^* \rangle \; .
\end{equation}

Consider any sequence $\{\eta_t\}_{t=1}^\infty$ of positive numbers approaching
$+\infty$.  Since $K$ is compact, $w(\eta_t)$ has a convergent subsequence.
Thus $\{w(\eta_t)\}_{t=1}^\infty$ has at least one accumulation point; let
$w^*$ be any of them. We will show that $w^* = v^*$.

Consider a subsequence $\{\xi_t\}_{t=1}^\infty$ of $\{\eta_t\}_{t=1}^\infty$
such that $\lim_{t \to \infty} w(\xi_t) = w^*$.  Since $w(\xi_t) \in S$ and $S$
is closed, $w^* \in S$.  From (\ref{equation:limit-equality}) we have $\langle
L, w^* \rangle = \langle L, v^* \rangle$ and hence $w^* \in H$. Thus $w^* \in S
\cap H$.  Since $v^*$ is the only point in $S \cap H$ we must have $w^* = v^*$.
\end{proof}

\begin{lemma}[Limit of Bregman Divergence]
\label{lemma:limit-bregman-divergence}
Let $K$ be a non-empty bounded closed convex subset of a finite dimensional
normed real vector space $(V, \|\cdot\|)$.  Let $R:K \to \R$ be a strongly
convex lower semi-continuous function bounded from above. Then, for any $x,y
\in V^*$,
$$
\lim_{a \to 0^+} a \Breg_{R^*}(x/a, y/a) = \langle x, u - v \rangle \\
$$
where
\begin{align*}
u = \lim_{a \to 0^+} \argmin_{w \in K} \left( a R(w) - \langle x, w \rangle \right)
\quad \text{and} \quad
v = \lim_{a \to 0^+} \argmin_{w \in K} \left( a R(w) - \langle y, w \rangle \right) \; .
\end{align*}
\end{lemma}

\begin{proof}
Using property 3 of Proposition~\ref{proposition:conjugate-properties} we can write
the divergence
\begin{align*}
a \Breg_{R^*}(x/a, y/a) & = a R^*(x/a) - a R^*(y/a) - \langle x - y, \grad R^*(y/a) \rangle \\
& =
 a \left[ \langle x/a, \grad R^*(x/a) \rangle - R(\grad R^*(x/a)) \right] \\
& \qquad - a \left[ \langle y/a, \grad R^*(y/a) \rangle - R(\grad R^*(y/a)) \right]
- \langle x - y, \grad R^*(y/a) \rangle \\
& =
\langle x, \grad R^*(x/a) - \grad R^*(y/a) \rangle - a R(\grad R^*(x/a))
+ a R(\grad R^*(y/a)) \; .
\end{align*}
Property 2 of Proposition~\ref{proposition:conjugate-properties} implies that
\begin{align*}
u = \lim_{a \to 0^+} \grad R^*(x/a) & = \lim_{a \to 0^+} \argmin_{w \in K} \left( a R(w) - \langle x, w \rangle \right) \; , \\
v = \lim_{a \to 0^+} \grad R^*(y/a) & = \lim_{a \to 0^+} \argmin_{w \in K} \left( a R(w) - \langle y, w \rangle \right) \; .
\end{align*}
The limits on the right exist according to
Lemma~\ref{lemma:prediction-limit-existence}.  They are simply the minimizers
$u = \argmin_{w \in K} - \langle x, w \rangle$ and $v= \argmin_{w \in K} -
\langle y, w \rangle$ where ties in $\argmin$ are broken according to smaller
value of $R(w)$.

By assumption $R(w)$ is upper bounded. It is also lower bounded, since it is
defined on a compact set and it is lower semi-continuous. Thus,
\begin{align*}
& \lim_{a \to 0^+} a \Breg_{R^*}(x/a, y/a) \\
& = \lim_{a \to 0^+} \langle x, \grad R^*(x/a) - \grad R^*(y/a) \rangle - a R(\grad R^*(x/a)) + a R(\grad R^*(y/a)) \\
& = \lim_{a \to 0^+} \langle x, \grad R^*(x/a) - \grad R^*(y/a) \rangle = \langle x, u - v \rangle \; .
\end{align*}
\end{proof}

\section{Proofs for \textsc{SOLO FTRL}}
\label{section:solo-ftrl-proofs}

\begin{proof}[Proof of Lemma~\ref{lemma:sum-of-square-roots-inverses}]
We use the inequality $x/\sqrt{x+y} \le 2(\sqrt{x+y} - \sqrt{y})$ which holds
for non-negative $x,y$ that are not both zero. Substituting $x = a_t$ and $y=\sum_{i=1}^{t-1} a_i$,
we get that for any $t \ge 1$,
$$
\frac{a_t}{\sqrt{\sum_{i=1}^t a_i}} \le 2 \sqrt{\sum_{i=1}^t a_i} \ - \ 2 \sqrt{\sum_{i=1}^{t-1} a_i} \; .
$$
Summing the above inequality over all $t=1,2,\dots,T$, the right side telescopes to
$2 \sqrt{\sum_{t=1}^T a_t}$.
\end{proof}
\section{Proofs for Scale-Free Mirror Descent}
\label{section:mirror-descent-proofs}

\begin{proof}[Proof of Lemma~\ref{lemma:mirror-descent-regret}]
Let
\begin{align*}
\Psi_{t+1}(w)
& = \langle \ell_t, w \rangle + \Breg_{R_t}(w,w_t) \\
& = \langle \ell_t, w \rangle + R_t(w) - R_t(w_t) - \langle \grad R_t(w_t) , w - w_t \rangle \; .
\end{align*}
Then, $w_{t+1} = \argmin_{w \in K} \Psi_{t+1}(w)$. Note that $\grad
\Psi_{t+1}(w) = \ell_t + \grad R_t(w) - \grad R_t(w_t)$. The optimality
condition for $w_{t+1}$ states that $\langle \grad \Psi_{t+1}(w_{t+1}), u -
w_{t+1} \rangle \ge 0$ for all $u \in K$. Written explicitly,
$$
\langle \ell_t + \grad R_t(w_{t+1}) - \grad R_t(w_t), u - w_{t+1} \rangle \ge 0 \; .
$$
Adding $\langle \ell_t, w_{t+1} - w_t \rangle$ to both sides and rearranging,
we have
\begin{align*}
\langle \ell_t, w_t - u \rangle
& \le \langle \grad R_t(w_{t+1}) - \grad R_t(w_t), u - w_{t+1} \rangle  + \langle \ell_t, w_t - w_{t+1} \rangle \\
& = \langle \ell_t, w_t - w_{t+1} \rangle - \Breg_{R_t}(w_{t+1}, w_t) + \Breg_{R_t}(u,w_t) - \Breg_{R_t}(u, w_{t+1}) \; .
\end{align*}
The last equality follows by from definition of Bregman divergence.  Summation
over all $t=1,2,\dots,T$ gives the final regret bound.
\end{proof}

\begin{proof}[Proof of Theorem~\ref{theorem:regret-scale-free-mirror-descent}]
Let $\eta_t = \frac{1}{\sqrt{\sum_{i=1}^t \|\ell_i\|_*^2}}$. We define $\eta_0
= +\infty$ and $1/\eta_0 = 0$.  Hence $R_t(w) = \frac{1}{\eta_t} R(w)$.  Since
$R_t$ is $\frac{\lambda}{\eta_t}$-strongly convex, we have
\begin{align*}
\langle \ell_t, w_t - w_{t+1} \rangle - \Breg_{R_t}(w_{t+1}, w_t)
& \le \|\ell_t\|_* \cdot \|w_t - w_{t+1}\| - \frac{\lambda}{2\eta_t} \|w_t - w_{t+1} \|^2 \\
& \le \max_{z \in \R} \left( \|\ell_t\|_* z - \frac{\lambda}{2 \eta_t} z^2 \right) \\
& = \frac{\eta_t}{2\lambda} \|\ell_t\|_*^2 \; .
\end{align*}
Combining the last inequality with Lemma~\ref{lemma:mirror-descent-regret}, we have
$$
\Regret_T(u) \le \sum_{t=1}^T \frac{\eta_t}{2\lambda} \|\ell_t\|_*^2 + \sum_{t=1}^T \left[ \Breg_{R_t}(u,w_t) - \Breg_{R_t}(u, w_{t+1}) \right] \; .
$$
Since $R_t(w) = \frac{1}{\eta_t} R(w)$, we have
\begin{align*}
\Regret_T(u)
& \le \frac{1}{2\lambda} \sum_{t=1}^T \eta_t \|\ell_t\|_*^2 + \sum_{t=1}^T \frac{1}{\eta_t} \left[\Breg_R(u, w_t) - \Breg_R(u, w_{t+1})\right] \\
& \le \frac{1}{2\lambda} \sum_{t=1}^T \eta_t \|\ell_t\|_*^2 + \sum_{t=1}^T \Breg_R(u, w_t) \left( \frac{1}{\eta_t} - \frac{1}{\eta_{t-1}} \right) \\
& \le \frac{1}{2\lambda} \sum_{t=1}^T \frac{\|\ell_t\|_*^2}{\sqrt{\sum_{i=1}^t \|\ell_i\|_*^2}} + \sup_{v \in K} \Breg_R(u,v) \sum_{t=1}^T \left( \frac{1}{\eta_t} - \frac{1}{\eta_{t-1}} \right) \\
& \le \frac{1}{\lambda} \sqrt{\sum_{t=1}^T \|\ell_t\|_*^2} + \sup_{v \in K} \Breg_R(u,v) \sqrt{\sum_{t=1}^T \|\ell_t\|_*^2} \qquad \text{(By Lemma~\ref{lemma:sum-of-square-roots-inverses})} \\
& = \left( \frac{1}{\lambda} + \sup_{v \in K} \Breg_R(u,v) \right) \sqrt{\sum_{t=1}^T \|\ell_t\|_*^2} \; .
\end{align*}
\end{proof}

\begin{proof}[Proof of Theorem~\ref{theorem:first-counter-example}]
We assume $d=1$. For $d \ge 2$, we simply embed the one-dimensional loss
vectors into the first coordinate of $\R^d$.  Consider the sequence
$$
(\ell_1, \ell_2, \dots, \ell_T) = ( \underbrace{- 1, -1, \dots, -1}_{\lceil T/2 \rceil}, \underbrace{+1, +1, \dots, +1}_{\lfloor T/2 \rfloor}) \; .
$$
The first half consists of $-1$'s, the second of $+1$'s. For $t \le \lceil T/2 \rceil$
$$
w_{t+1} = w_t + \frac{1}{\sqrt{t}} \; .
$$
Unrolling the recurrence and using $w_1 = 0$ we get
$$
w_t = \sum_{i=1}^{t-1} \frac{1}{\sqrt{i}} \qquad \text{(for $t \le \lceil T/2 \rceil + 1$)} \; .
$$
On the other hand, for $t \ge \lceil T/2 \rceil + 1$, we have
$$
w_{t+1} = w_t - \frac{1}{\sqrt{t}} \; .
$$
Unrolling the recurrence up to $w_{\lceil T/2 \rceil + 1}$ we get
$$
w_t
= w_{\lceil T/2 \rceil + 1} \ \ - \sum_{i=\lceil T/2 \rceil + 1}^{t-1} \frac{1}{\sqrt{i}}
= \sum_{i=1}^{\lceil T/2 \rceil} \frac{1}{\sqrt{i}} \ \ - \sum_{i=\lceil T/2 \rceil + 1}^{t-1} \frac{1}{\sqrt{i}}
\qquad \text{(for $t \ge \lceil T/2 \rceil + 1$)} \; .
$$
We are ready to lower bound the regret.
\begin{align*}
\Regret_T(0)
& = \sum_{t=1}^T \ell_t w_t \\
& = - \sum_{t=1}^{\lceil T/2 \rceil} w_t + \sum_{t=\lceil T/2 \rceil + 1}^T w_t  \\
& = - \sum_{t=1}^{\lceil T/2 \rceil} \sum_{i=1}^{t-1} \frac{1}{\sqrt{i}} \ \ + \sum_{t=\lceil T/2 \rceil + 1}^T \left( \sum_{i=1}^{\lceil T/2 \rceil} \frac{1}{\sqrt{i}} \ - \sum_{i=\lceil T/2 \rceil + 1}^{t-1} \frac{1}{\sqrt{i}} \right) \\
& = - \sum_{i=1}^{\lceil T/2 \rceil} \frac{\lceil T/2 \rceil - i}{\sqrt{i}} \ + \ \lfloor T/2 \rfloor \sum_{i=1}^{\lceil T/2 \rceil} \frac{1}{\sqrt{i}} \ \ - \sum_{i=\lceil T/2 \rceil + 1}^T \frac{T - i}{\sqrt{i}} \\
& = - \sum_{i=1}^{\lceil T/2 \rceil} \frac{\lceil T/2 \rceil - \lfloor T/2 \rfloor}{\sqrt{i}} \ + \ \sum_{i=1}^{T} \sqrt{i} \ \ - \ T \!\!\!\! \sum_{i=\lceil T/2 \rceil + 1}^T \frac{1}{\sqrt{i}} \\
& \ge - \sum_{i=1}^{\lceil T/2 \rceil} \frac{1}{\sqrt{i}} \ \ + \ \sum_{i=1}^{T} \sqrt{i} \ \ - \ T \!\!\!\! \sum_{i=\lceil T/2 \rceil + 1}^T \frac{1}{\sqrt{i}} \\
& \ge - 1 - \int_{i=1}^{\lceil T/2 \rceil} \frac{1}{\sqrt{x}} dx \ + \ \int_{0}^T \!\!\! \sqrt{x} \, dx \ \ - \ T \int_{\lceil T/2 \rceil}^T \frac{1}{\sqrt{x}} dx \\
& = - 1 - 2\left(\sqrt{\lceil T/2 \rceil} - 1 \right) + \frac{2}{3} T^{3/2} - 2 T \left(\sqrt{T} - \sqrt{\lceil T/2 \rceil} \right) \\
& \ge 1 - 2 \sqrt{\lceil T/2 \rceil} + \left( \frac{2}{3} - 2 + \sqrt{2} \right) T^{3/2} \; .
\end{align*}
The last expression is $\Omega(T^{3/2})$ with dominant term $(\frac{2}{3} - 2 +
\sqrt{2}) T^{3/2} \approx 0.08 \cdot T^{3/2}$.  For any $T \ge 42$, the
expression is lower bounded by $\frac{1}{20} T^{3/2}$.
\end{proof}

\begin{proof}[Proof of Theorem~\ref{theorem:second-counter-example}]
Let $e_1, e_2, \dots, e_d$ be the standard orthnormal basis of $\R^d$. Consider
the sequence of loss vectors
$$
(\ell_1, \ell_2, \dots, \ell_T) = ( \underbrace{-e_1, -e_1, \dots, -e_1}_{\lceil T/3 \rceil}, \underbrace{-e_2, -e_2, \dots, -e_2}_{\lfloor 2T/3 \rfloor}) \; .
$$
First, for any $t \ge \lceil T/3 \rceil + 1$,
\begin{align*}
\frac{w_{t,1}}{w_{t,2}}
& = \frac{\exp(-\sum_{i=1}^{t-1} \ell_{i,1}/\sqrt{i})}{\exp(-\sum_{i=1}^{t-1} \ell_{i,2}/\sqrt{i})} \\
& = \frac{\exp(\sum_{i=1}^{\lceil T/3 \rceil} 1/\sqrt{i})}{\exp(\sum_{i=\lceil T/3 \rceil + 1}^t 1/\sqrt{i})} \\
& \ge \frac{\exp(\sum_{i=1}^{\lceil T/3 \rceil} 1/\sqrt{i})}{\exp(\sum_{i=\lceil T/3 \rceil + 1}^T 1/\sqrt{i})} \\
& = \exp\left(\sum_{i=1}^{\lceil T/3 \rceil} \frac{1}{\sqrt{i}} \ \ - \ \sum_{i=\lceil T/3 \rceil + 1}^T 1/\sqrt{i} \right) \\
& \ge \exp\left( \int_{1}^{\lceil T/3 \rceil + 1} \frac{dx}{\sqrt{x}} \ \ - \ \int_{\lceil T/3 \rceil}^T \frac{dx}{\sqrt{x}} \right) \\
& = \exp\left( 2\sqrt{\lceil T/3 \rceil + 1} - 2 - (2\sqrt{T} - 2 \sqrt{\lceil T/3 \rceil}) \right) \\
& \ge \exp\left( \left( \frac{4}{\sqrt{3}} - 2 \right) \sqrt{T} - 2 \right) \\
& \ge 4 \; ,
\end{align*}
where the last inequality follows from the fact that $\exp\left( \left(
\frac{4}{\sqrt{3}} - 2 \right) \sqrt{T} - 2 \right)$ is an increasing function
of $T$ and the inequality can be easily verified for $T=120$. Since $w_{t,1} +
w_{t,1} \le 1$ and $w_{t,1} \ge 0$ and $w_{t,2} \ge 0$, the inequality
$w_{t,1}/w_{t,2} \ge 4$ implies that
$$
w_{t,2} \le \frac{1}{5} \qquad \qquad \text{(for any $t \ge \lceil T/3 \rceil + 1$)} \; .
$$
Now, we lower bound the regret. Since $T \ge 120$,
\begin{align*}
\Regret_T
& \ge \Regret_T(e_2) \\
& = \sum_{t=1}^T \langle \ell_t, w_t \rangle - \sum_{t=1}^T \langle \ell_t, e_2 \rangle \\
& = - \sum_{t=1}^{\lceil T/3 \rceil} w_{t,1} \ \ - \ \sum_{t=\lceil T/3 \rceil + 1}^T w_{t,2}  +  \lfloor 2T/3 \rfloor \\
& \ge - \lceil T/3 \rceil - \frac{1}{5} \lfloor 2T/3 \rfloor  +  \lfloor 2T/3 \rfloor \\
& \ge - T/3 - 1 - 2T/15 + 2T/3 - 1 \\
& = T/5 - 2 \\
& \ge T/6 \; .
\end{align*}
\end{proof}

\section{Lower Bound Proof}
\label{section:lower-bound-proof}

\begin{proof}[Proof of Theorem~\ref{theorem:simple-lower-bound}]
Pick $x,y \in K$ such that $\|x - y\| = D$. This is possible since $K$ is compact.
Since $\|x - y\| = \sup \{\langle \ell, x - y \rangle ~:~ \ell \in V^*, \|\ell\|_* = 1\}$
and the set $\{ \ell \in V^* ~:~ \|\ell\|_* = 1 \}$ is compact, there exists $\ell \in V^*$
such that
$$
\|\ell\|_* = 1 \qquad \text{and} \qquad \langle \ell, x - y \rangle = \|x - y\| = D \; .
$$
Let $Z_1, Z_2, \dots, Z_T$ be i.i.d. Rademacher variables, that is,
$\Pr[Z_t = +1] = \Pr[Z_t = -1] = 1/2$. Let $\ell_t = Z_t a_t \ell$.
Clearly, $\|\ell_t\|_* = a_t$. The lemma will be proved if we show that
(\ref{equation:simple-lower-bound}) holds with positive probability.
We show a stronger statement that the inequality holds in expectation, i.e.,
$\Exp[\Regret_T] \ge \frac{D}{\sqrt{8}} \sqrt{\sum_{t=1}^T a_t^2}$. Indeed,
\begin{align*}
\Exp\left[ \Regret_T \right]
& \ge \Exp\left[ \sum_{t=1}^T \langle \ell_t, w_t \rangle \right] - \Exp\left[\min_{u \in \{x,y\}} \sum_{t=1}^T \langle \ell_t, u \rangle \right] \\
& = \Exp\left[ \sum_{t=1}^T Z_t a_t \langle \ell, w_t \rangle \right] + \Exp \left[\max_{u \in \{x,y\}} \sum_{t=1}^T -Z_t a_t \langle \ell, u \rangle \right]  \\
& = \Exp\left[ \max_{u \in \{x,y\}} \sum_{t=1}^T -Z_t a_t \langle \ell, u \rangle \right] \\
& = \Exp\left[ \max_{u \in \{x,y\}} \sum_{t=1}^T Z_t a_t \langle \ell, u \rangle \right]  \\
& = \frac{1}{2} \Exp\left[ \sum_{t=1}^T Z_t a_t \langle \ell, x + y \rangle \right]  + \frac{1}{2}\Exp\left[ \left|\sum_{t=1}^T Z_t a_t \langle \ell, x - y \rangle \right| \right] \\
& = \frac{D}{2}\Exp\left[ \left|\sum_{t=1}^T Z_t a_t \right| \right] \\
& \ge \frac{D}{\sqrt{8}} \sqrt{\sum_{t=1}^T a_t^2} \; ,
\end{align*}
where we used that $\Exp[Z_t] = 0$, the fact that distributions of $Z_t$ and
$-Z_t$ are the same, the formula $\max\{a,b\} = (a+b)/2 + |a-b|/2$, and
Khinchin's inequality in the last step (Lemma A.9 in
\cite{Cesa-Bianchi-Lugosi-2006}).
\end{proof}

\end{document}